%% file: iclr2027_conference.tex
\documentclass{article} %
\usepackage{iclr2027_conference,times}

\input{math_commands.tex}

\usepackage{xcolor}
\definecolor{cmbpurple}{RGB}{123,50,148} %
\usepackage{xspace}
\newcommand{\CLIMB}{CLIMB\xspace}
\usepackage{booktabs}
\usepackage{colortbl}
\usepackage{multirow}
\usepackage{makecell}
\usepackage{graphicx}
\usepackage{float}
\usepackage{subcaption}

\usepackage{tikz}
\usetikzlibrary{arrows.meta,positioning,shapes.geometric,calc,fit,backgrounds}
\usepackage{algorithm}
\usepackage{algpseudocode}
\definecolor{okgreen}{RGB}{46,139,87}
\definecolor{badorange}{RGB}{230,145,56}
\definecolor{badred}{RGB}{200,60,70}
\usepackage{hyperref}
\usepackage{url}
\usepackage{amsthm}

\usepackage{amsmath,amssymb,amsthm}

\theoremstyle{plain}
\newtheorem{proposition}{Proposition}
\newtheorem{lemma}{Lemma}

\theoremstyle{definition}
\newtheorem{definition}{Definition}
\newtheorem{assumption}{Assumption}

\theoremstyle{remark}

\usepackage{listings}
\lstdefinestyle{prompt}{%
  basicstyle=\ttfamily\scriptsize,
  breaklines=true,
  columns=flexible,
  keepspaces=true,
  showstringspaces=false,
  frame=single,
  rulecolor=\color{black!30},
  backgroundcolor=\color{black!4},
  xleftmargin=6pt,
  xrightmargin=2pt,
  aboveskip=3pt,
  belowskip=8pt,
}

\newif\ifcommentson
\commentsonfalse %

\usepackage{verbatim}

\title{\raggedright\hyphenpenalty=10000 \textbf{CLIMB}: A \textbf{CLI}nical \textbf{M}ultimorbidity \textbf{B}enchmark for diagnosing co-occurring conditions through multiturn conversations}

\author{
\mbox{Yusuf Kesmen$^{1}$\thanks{Corresponding author: \texttt{yusuf.kesmen@epfl.ch}} \quad
Aniruddha Mukherjee$^{1}$\thanks{Equal contribution.} \quad
Yena Chang$^{1\dagger}$} \\
\mbox{\textbf{David Sasu}$^{1}$ \quad
\textbf{Trevor Brokowski}$^{1}$ \quad
\textbf{Alexandra V.~Kulinkina}$^{1}$ \quad
\textbf{Kristina Keitel}$^{1}$} \\
\mbox{\textbf{Akhil Arora}$^{2}$\thanks{Equal supervision.} \quad
\textbf{Lars Henning Klein}$^{1\ddagger}$ \quad
\textbf{Mary-Anne Hartley}$^{1\ddagger}$} \\[5pt]
\mbox{$^{1}$EPFL \quad
$^{2}$Aarhus University}
}

\iclrfinalcopy %
\begin{document}

\maketitle

\begin{abstract}
\input{sections/abstract.tex}

\end{abstract}

\input{sections/introduction.tex}

\begin{figure}[t]
    \centering
    \includegraphics[width=0.82\linewidth, trim={0cm 2.6cm 1.5cm 1cm},
  clip]{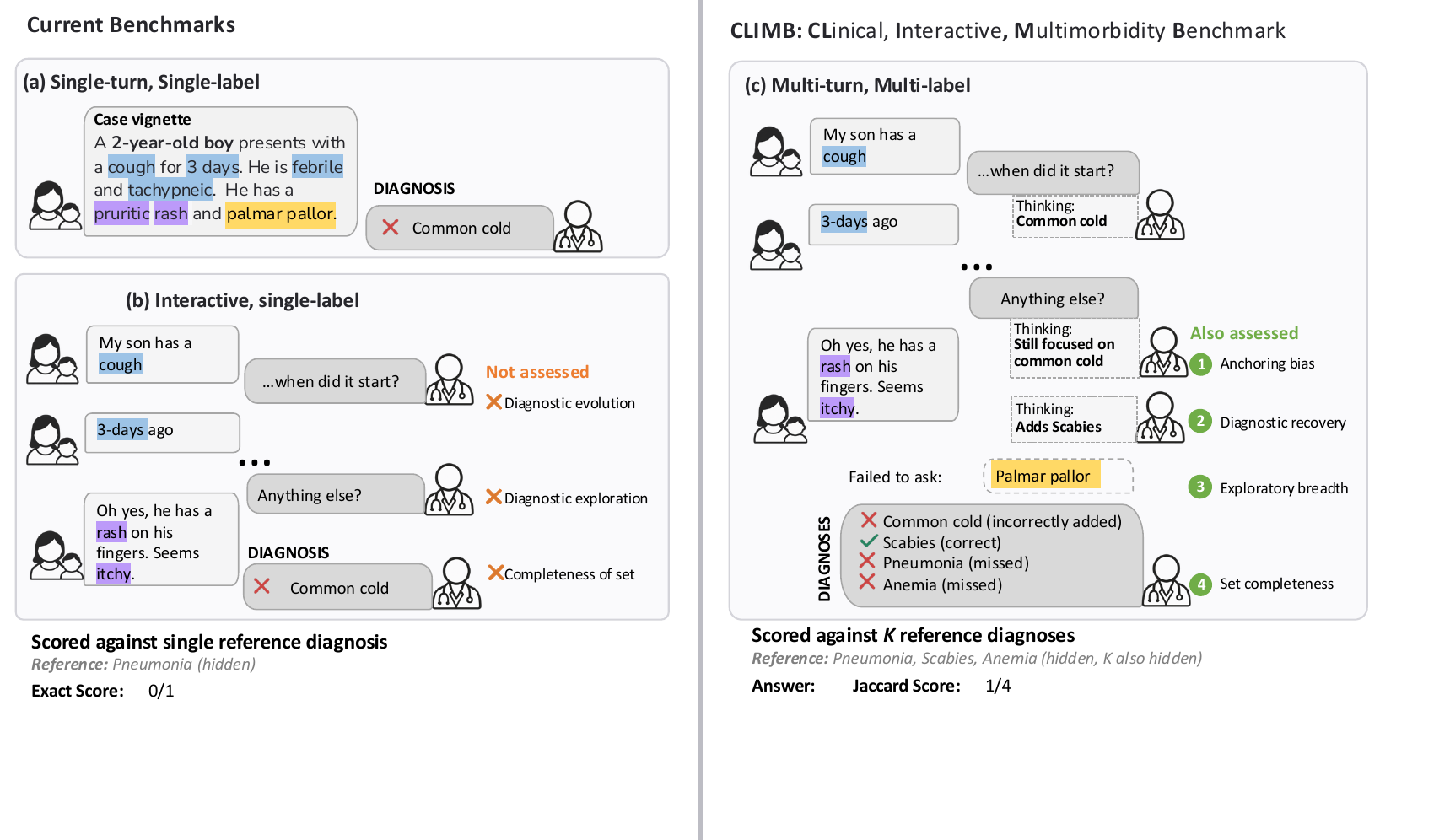}
    \caption{Static and single-label settings against ours. Panels (a) and (b)
    provide all evidence upfront or elicit it for a single correct diagnosis, while
    \CLIMB (c) elicits evidence over a multi-turn dialogue for a patient whose
    ground truth is a set of co-occurring diagnoses, so the model must hold several
    hypotheses at once.
    }
    \label{fig:overview}
\vspace{-2mm}
\end{figure}

\input{sections/related_work.tex}
\input{sections/method/method.tex}

\input{sections/experiments/experiments.tex}

\input{sections/results/results.tex}

\input{sections/failure/failure.tex}

\input{sections/discussion.tex}

\clearpage
\subsection*{AI use statement}

We used generative AI tools to write and refactor code for our pipeline, to generate parts
of the synthetic patient data statistics released with this work, and to polish writing. The generated data were inspected by the authors for clinical plausibility. Writing assistance was restricted to phrasing and concision; all claims, citations, and
numerical values were written and checked by the authors against their sources.

We did not use generative AI tools for research ideation. Language models additionally appear in this
paper as objects of study and as components of our evaluation setup. Those uses are
methodological and are described in the corresponding sections.

We have reviewed all AI-assisted work and take responsibility for the final content of this
paper, including text, claims, and artifacts produced with the aid of generative AI.

\subsection*{Ethics statement}

This work involves no human subjects and no patient-level data. Our synthetic patients are
derived from two sources. The first is DDXPlus, which is itself fully synthetic. The second
is ePOCT+, a clinical decision support algorithm whose decision tree was authored and
validated by clinicians; from it we use the tree structure together with aggregate disease
priors and conditional symptom probabilities. These quantities were estimated from cohorts
of real patients managed under ePOCT+, but no individual records enter our pipeline and no
identifiable or otherwise sensitive information is used or released.

The benchmark is intended for evaluating language models. This is not for clinical use. Also, performance on this
benchmark should therefore not be read as evidence of clinical safety or of readiness for
deployment.

\subsection*{Reproducibility statement}

Section~\ref{sec:method} and Appendix~\ref{app:generation} specify the source
resources, association models, exclusion masks, sampling procedures, and
merging and rejection rules used to construct cases. The accompanying
data-profile supplement provides source-file hashes and the script used to
compute the dataset summaries; Appendix~\ref{app:comor_population}
distinguishes the generated populations and evaluated subsets.
Section~\ref{sec:experiments} and Appendix~\ref{app:experiments} describe case
selection, model identifiers, decoding settings, question budgets, and
evaluation conditions. The implemented prompt templates and their adaptations
appear in Appendix~\ref{app:sim:prompts}, and scoring, diagnostic trajectories,
and uncertainty estimates are detailed in Appendix~\ref{app:comor_analysis}
and the subsequent analysis definitions. Evaluation logs retain model and task
settings, consultation transcripts, and predictions, allowing metrics to be
recomputed from stored outputs separately from rerunning model calls. The benchmark, generator, and evaluation
code are available at \url{https://anonymous.4open.science/r/CLIMB-8340}.

\subsubsection*{Acknowledgments}
This work was supported by the EPFL AI Center PhD Fellowship Program 2026.
The LiGHT lab is partly supported by the Gates Foundation (INV-076674), the Kristian
Gerhard Jebsen Foundation, Swisscom (Schweiz) AG through the Swiss National AI
Institute Industry Partnership Program, and the ETH Domain through the Swiss AI
Initiative. Support from the Swiss AI Initiative includes computational resources on
the Swiss National Supercomputing Centre (CSCS) Alps infrastructure (project IDs 27
and a0238) and a PhD fellowship for the first author.
Arora's lab is partly supported by grants from the Novo Nordisk Foundation
(NNF24OC0099109), the Pioneer Centre for AI, EU Horizon 2020 (101168951), and
Danish Advanced Research Academy (DARA) and Danish Data Science Academy (DDSA)
PhD Fellowships. We also gratefully acknowledge generous gifts from Microsoft and
IT-vest networking universities.

\bibliography{iclr2027_conference}
\bibliographystyle{iclr2027_conference}

\appendix
\clearpage
\input{appendix/extended_related_work.tex}
\clearpage
\input{appendix/patient_generation/patient_generation.tex}

\clearpage
\input{appendix/comorbidity.tex}
\clearpage
\input{appendix/experiments/experiments.tex}

\clearpage
\input{appendix/additional_results.tex}

\clearpage
\input{appendix/theory.tex}

\clearpage

\end{document}

%% file: math_commands.tex
\usepackage{amsmath,amsfonts,bm}

\def\eqref#1{equation~\ref{#1}}

\def\1{\bm{1}}

\DeclareMathAlphabet{\mathsfit}{\encodingdefault}{\sfdefault}{m}{sl}
\SetMathAlphabet{\mathsfit}{bold}{\encodingdefault}{\sfdefault}{bx}{n}

%% file: sections/abstract.tex
Patients often have several co-occurring clinical conditions, and the findings needed to identify and disambiguate them emerge over the course of a consultation. Evaluating clinical reasoning in this setting requires both multi-turn interaction and multi-label diagnosis. We introduce \CLIMB, a benchmark in which a doctor model interviews a simulated patient to recover a ground truth set
of co-occurring clinical conditions. Cases are synthesized from clinical decision algorithms and diagnostic datasets, grounding multimorbid presentations in structured clinical knowledge.  Across six frontier and open models, none
recovers the exact set of conditions in more than 10\% of interactive cases. Diagnostic performance declines when conditions co-occur,
even when models receive the full clinical record and the true number of conditions.
Interaction reduces performance further. In controlled experiments, models behave like single-hypothesis trackers: they anchor on the diagnosis suggested by the opening findings, keep questioning around it, and recover a second condition mainly when a finding in view points to it. Questioning them further does not complete the set but adds mostly wrong diagnoses. We formalise this pattern with a theoretical reference model of single-hypothesis tracking. The benchmark, generator, and evaluation
code are available at \url{https://anonymous.4open.science/r/CLIMB-8340}.

%% file: sections/introduction.tex
\section{Introduction}
\label{sec:introduction}

Clinical benchmarks for large language models typically assume a one--shot, single--diagnosis scenario in which the evidence is given upfront and the model is evaluated on its diagnostic prediction. In a realistic consultation, diagnosis is a sequential decision problem in which only partial information is initially available, and the clinician must decide what further information to acquire as evidence accumulates. Evaluating this process requires distinguishing diagnostic recognition from the ability to acquire and integrate evidence to identify multiple co-occurring conditions.

Language models are increasingly applied to clinical decision support~\citep{thirunavukarasu2023large, moor2023foundation}. A growing set of
medical models and diagnostic agents targets this
domain \citep{chen2023meditron,tu2025conversational}, demonstrating strong performance on
one--shot diagnostic predictions by
drawing on substantial encoded medical knowledge \citep{singhal2023large}.

Strong benchmark performance does not necessarily translate into reliable clinical judgment. An evaluation of ChatGPT Health, which reached millions of users within months of launch, found that it under-triaged 52\% of the cases that physician consensus judged to require emergency care ~\citep{ramaswamy2026chatgpt}. In deployment, these errors could translate into patient harm. This gap between benchmark performance and deployed behaviour highlights the need to assess models under the conditions of actual clinical consultation. This makes it important to not only correctly understand the current capabilities of frontier models, but also to evaluate them under the conditions that distinguish a real consultation from a simple classification. Two key dimensions distinguish real consultations from conventional static question answering benchmarks. 

On the first, unlike static benchmarks where all information is presented at once, clinicians must acquire information sequentially over the course of a consultation, which we refer to as the \textit{multi-turn} setting~\citep{li2024mediq}. On the second, unlike static Q\&A benchmarks where there is one correct diagnosis, a patient can have several conditions that hold at once, which we refer to as the \textit{multi-label} setting, or \textit{multimorbidity} in clinical
language~\citep{valderas2009defining,barnett2012epidemiology,dacremont2014beyond}. Interactive patient simulators have
been built to evaluate models in the multi-turn settings~\citep{li2024mediq,schmidgall2024agentclinic,chiu2025vivabench}, showing that language models demonstrate significant losses in accuracy as a conversation length increases~\citep{laban2025lost}.

The multi-label axis has drawn comparatively less attention. A few studies examine the
multi-label behavior in isolation \citep{ma2025multilabel}. Our
focus is the intersection of these two dimensions. We study diagnosis where (1) evidence
must be elicited through an interactive consultation, and (2) the ground truth consists of a set of co-occurring conditions. We investigate whether a language model can hold several simultaneous hypotheses, and acquire relevant evidence to revise them through an interactive consultation.

To our knowledge, no prior benchmark provides a controlled setting and analysis for interactive diagnosis of multimorbid patients. Prompting a frontier model to generate
patient cases offers little control over the resulting distribution of presentations, and
existing case notes rarely contain controlled variation in multimorbidity. We propose \CLIMB, a benchmark built on a principled procedure that generates multi-label synthetic
patients grounded in clinical source evidence, with an evaluation protocol tailored to multi-turn, multi-label diagnostic reasoning.

We make three contributions to the study of interactive multimorbidity diagnosis:
\begin{enumerate}
    \item \CLIMB, a controlled benchmark for multi-turn, multi-label diagnosis with synthetic multimorbid patients grounded in established clinical sources.
    \item A systematic comparison of diagnostic performance of 
    models for different numbers of co-occurring conditions, and information regime.
    \item  A failure analysis of this regime, with controlled experiments to understand the models' behaviour and failures, and a formal description of single-hypothesis tracking.
\end{enumerate}

%% file: sections/related_work.tex
\section{Related Work}
\label{sec:related_work}

We summarize the most relevant work here and give a broader treatment in
Appendix~\ref{app:extended_related_work}.

\paragraph{Static and interactive benchmarks.}
Multiple-choice medical exam benchmarks and their variants established that language
models carry a large amount of medical knowledge, and instruction-tuned systems now
reach strong scores on these formats \citep{jin2021medqa,pal2022medmcqa,jin2019pubmedqa,singhal2023large}. These benchmarks place all of the relevant information in the prompt and admit a single correct answer, and therefore mainly measure the retrieval of stored knowledge instead of how to gather the missed
evidence. On the other hand, interactive and multi-turn benchmarks evaluate a model's ability to gather information before it commits to an answer. These
works pair a doctor model with a simulated patient and score the final diagnosis
\citep{johri2025evaluation,schmidgall2024agentclinic, nori2025sequential}. Patient simulators support the setting by producing
person-driven patient behavior \citep{kyung2025patientsim}. This literature reports
recurring failure modes, where models under-ask, pose non-specific questions, and
close the consultation early \citep{laban2025lost}, which motivates interventions
built around uncertainty estimation and abstention \citep{kirichenko2025abstentionbench,hu2024uncertainty}.

\paragraph{Multimorbidity and diagnostic error.}
Multimorbidity, the presence of two or more conditions in one person, affects more
than half of people aged 65 and over \citep{barnett2012epidemiology}, is growing in
low and middle income countries \citep{asogwa2022multimorbidity}, and is associated
with higher mortality and polypharmacy~\citep{nunes2016multimorbidity}. We use
multimorbidity and comorbidity interchangeably (Appendix~\ref{app:comorbidity}).
Decision support for diagnosis dates back to early probabilistic systems
\citep{miller1982internist,shwe1991probabilistic}, and satisfaction of search,
calling off the search once one diagnosis is found, is a recognized source of
missed coexisting conditions \citep{croskerry2003importance}. Recent work trains a query policy for multi-disease consultations \citep{guo2026drassistant} or adds a
comorbid presentation to an interactive psychiatric benchmark
\citep{xu2026lingxidiagbench}, but does not study the co-occurrence itself. We
combine both dimensions. Evidence is gathered over a consultation, cases have a
controlled number of conditions, and the full predicted set is scored, which lets
us isolate how a second diagnosis is lost.

%% file: sections/method/method.tex
\section{Method}
\label{sec:method}
\paragraph{Objective.} 
Each case describes one patient. The patient has a set of $k$ co-occurring
conditions $\mathcal{D} \subset \mathcal{L}$, where $\mathcal{L}$ is the
set of candidate conditions and $k = |\mathcal{D}|$. The patient record
$x \subset \mathcal{F}$ is a set of findings; each finding $f \in \mathcal{F}$
pairs a clinical variable with its value, such as \emph{cough: yes}. In an interactive
consultation, the doctor model observes $\mathcal L$ and an opening presentation $H_0$ (age, sex, and partial observation $O \subset x$), while the target set $\mathcal D$, its
cardinality $k$, and the remaining findings are hidden. The objective is
to recover $\mathcal D$.

\CLIMB builds cases from ePOCT+
\citep{tan2023epoct} and DDXPlus \citep{fansitchango2022ddxplus} by repeatedly sampling a target set of size $k$ until its realized set of components (the findings that its source dataset links it to) passes consistency checks after merging. Rejected proposals are resampled, and accepted \emph{signals}
$S \subset \mathcal F$ are mixed with background complaints $B \subset \mathcal F$ (Figure~\ref{fig:pipeline}).

\begin{figure}[t]
    \centering
    \includegraphics[width=\linewidth, trim={0cm 4cm 0.5cm 3cm},
  clip]{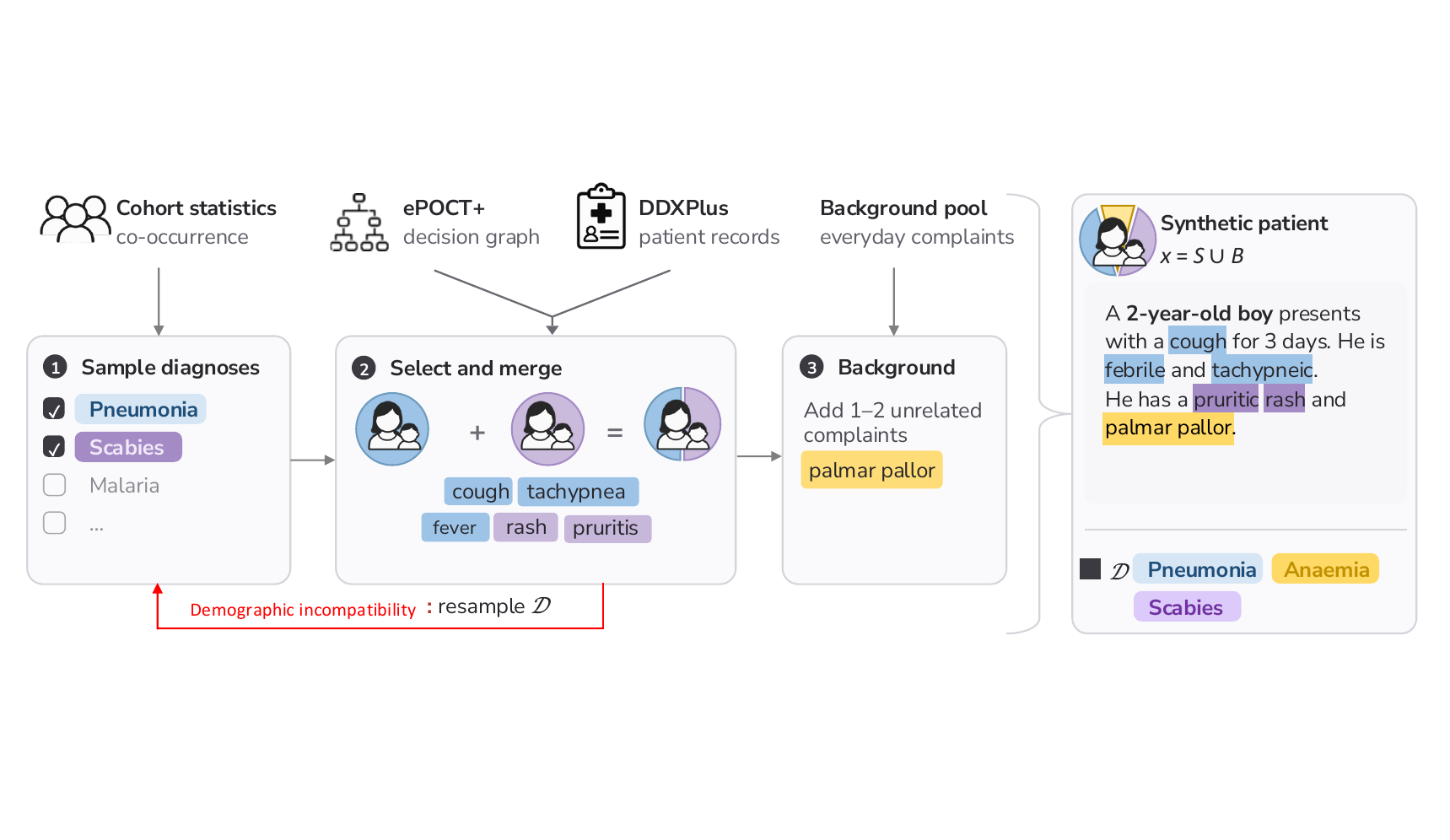}
    \caption{Case construction. Diagnoses are sampled from cohort
 co-occurrence, their findings are taken from ePOCT+ paths or DDXPlus records
 and merged, and one or two unrelated complaints are added. A conflict
 resamples the diagnosis set.}
    \label{fig:pipeline}
\vspace{-2mm}
\end{figure}

\paragraph{Comorbidity sampler.}
Both sources draw the $k$ targets sequentially without replacement with
\begin{equation}
 p(d\mid R)\propto\widehat P(d)
 \prod_{d'\in R}\ell(d,d'),
 \qquad d\notin R,
 \label{eq:sampler}
\end{equation}
where $R$ is the set of targets already drawn, $\widehat P$ the empirical prevalence, and $\ell(d,d')>1$ when $d$ and $d'$ co-occur more often than chance. The first draw uses
prevalence alone.

\input{sections/method/epoct.tex}

\input{sections/method/ddx.tex}

\input{sections/method/background.tex}

%% file: sections/method/epoct.tex
\subsection{Case Generation from ePOCT+}
\label{sec:epoct}
ePOCT+ is a pediatric clinical decision algorithm represented by a
graph $\mathcal T$ of clinical variables and diagnosis rules.
Appendix~\ref{app:source_epoct} introduces the source graph and cohort
statistics; Appendix~\ref{app:comor_population} profiles the generated cases.

\paragraph{Statistics and masks.}
Confirmed cohort diagnoses provide $\widehat P(d)=c_d/n$ and
$\widehat P(d,d')=c_{dd'}/n$, giving pairwise association factors
$\ell(d,d')=\widehat P(d,d')/[\widehat P(d)\widehat P(d')]$.
An offline LLM screen proposes infeasible pairs that are verified by our clinicians. We
apply this label mask and retain diagnoses represented in both the
cohort and graph.

\paragraph{Components, merging, and rejection.}
For each sampled $d$, we enumerate its root-to-diagnosis paths and discard
treatment and referral conditions. Several answers listed for one variable
are alternatives, and a path is retained if every variable keeps at least
one answer and the path has a visible finding. A uniformly selected
retained path $\sigma(d)$, with one allowed answer per variable, defines
$c(\sigma(d))$. The merged signal is
\begin{equation}
 S=\bigcup_{d\in\mathcal D}c(\sigma(d)).
 \label{eq:signal}
\end{equation}
We reject the proposal if a diagnosis has no retained path, shared
variables have conflicting answers, the mask excludes a diagnosis pair,
age windows do not overlap, a diagnosis has no visible finding of its own,
or fewer than two visible findings remain. We then resample $\mathcal D$ at the same $k$. For accepted
cases, demographics follow the compatible age window, and the case
receives an opening presentation. Cohort processing, path extraction,
and rejection are detailed in Appendix~\ref{app:comor_epoct}; the
label-masking prompt is in Appendix~\ref{app:comor_masks}.

%% file: sections/method/ddx.tex
\subsection{Case Generation from DDXPlus}
\label{sec:ddx}
DDXPlus provides synthetic patient records with a single target condition, symptoms, medical-history evidence, and a differential diagnosis.
Appendix~\ref{app:source_ddx} describes its scope, record space, and
evidence types; we profile our merged cases in Appendix~\ref{app:comor_population}.

\paragraph{Statistics and masks.}
We estimate prevalence from the training split. Since each record has a single condition, the co-occurrence
ratios $\ell(d,d')$ are elicited from an LLM. Separately, a second LLM screen proposes infeasible pairs, which are verified by our clinicians. During sampling, we apply a mask which excludes candidates that form an infeasible pair with a drawn target, or have no overlapping age window. 
The co-occurrence model, including its conversion to bounded pairwise joint probabilities, is specified in Appendix~\ref{app:comor_ddx}; the masking prompt and
verdict rules are in Appendix~\ref{app:comor_masks}.

\paragraph{Components, merging, and rejection.}
For every sampled diagnosis, we select a source record with a common
sex and an age near an anchor in the shared age window. Binary and
multi-select findings are merged by union, ordered values by maximum,
and conflicting nominal values cause rejection. We retry record
selection. If no compatible realization is found, we resample the
diagnosis set at the same $k$. Cases in which a diagnosis keeps no
finding of its own, or with fewer than two findings, are also rejected. Accepted findings are rendered as question-answer
statements, with the highest-severity component supplying the opening
complaint (Appendix~\ref{app:comor_ddx}).

%% file: sections/method/background.tex
\subsection{Background Complaints and Final Records}
\label{sec:background}

Real patients also report complaints unrelated to their
conditions~\citep{elnegaard2015symptom}, so both generators add background
complaints $B$ to the accepted signal,
\begin{equation}
 N\sim\mathcal{U}\{1,2\},\qquad
 B\sim p_{\mathrm{bg}}(\cdot\mid N,s),\qquad x=S\cup B,
 \label{eq:patient}
\end{equation}
drawn without replacement with the sex-specific symptom prevalences of
\citet{elnegaard2015symptom}, where $s$ is the recorded sex. ePOCT+ keeps only
complaints applicable to children. The count is small and independent of $k$,
so irrelevant evidence cannot explain differences across $k$ (Appendix~\ref{app:comor_background}).

%% file: sections/experiments/experiments.tex
\section{Experiments}
\label{sec:experiments}

\input{sections/experiments/data.tex}
\input{sections/experiments/protocol.tex}

%% file: sections/experiments/data.tex
\subsection{Data and Models}
\label{sec:experiments:data}
\label{sec:experiments:setup}

Our evaluation protocol uses all 800 stored cases per source, balanced with
200 cases at each $k\in\{1,2,3,4\}$, and holds cases fixed across
models and conditions. The cases are the released ePOCT+ and DDXPlus
population files (Appendix~\ref{app:comor_population}). The oracle-count and
single-diagnosis controls and the probe analyses of
Section~\ref{sec:stopping} use the first 50 of these cases at each $k$. Every condition of the main evaluation uses the
complete record $x=S\cup B$: the full-information prompt lists all of
its findings, and the patient simulator answers from the same record
(Section~\ref{sec:background}).

For the main experiments we use GPT-5.6, Gemini-3.8-Flash,
Qwen3.8-Flash, Gemma-4, GLM-5.3-Flash, and DeepSeek-V4.1-Flash as doctor models (Appendix~\ref{app:model_selection}).
Free-text consultations use DeepSeek-V4.1-Flash as the patient,
chosen by screening candidates for answering from the record without
inventing findings or naming the diagnosis (Appendix~\ref{app:patient_screen}).
The patient is fixed across all doctors and conditions, so it cannot explain
differences between them. The patient model receives the
opening presentation and the complete finding record $x$ without target labels. GPT-5.6 is used for the failure
analyses of Section~\ref{sec:failure}. Model identifiers, reasoning and
decoding settings, and the patient screen are detailed in
Appendix~\ref{app:comor_setup}.

%% file: sections/experiments/protocol.tex
\subsection{Evaluation Protocol}
\label{sec:protocol}

\paragraph{Information conditions.} \emph{Full information} supplies the opening and all retained findings
upfront. \emph{Interactive} evaluation uses an 
doctor/patient dialogue derived from AgentClinic \citep{schmidgall2024agentclinic}, with a
budget of $T=20$ question turns. The patient uses the standard
AgentClinic patient prompt with the case record as its information
(Appendix~\ref{app:sim:deviations}). Each turn, the doctor asks a question
or commits to a diagnosis set. Both conditions request all supported diagnoses
from the same candidate set and use the same scorer. We score the first committed set,
or request a final set when the consultation ends without a commitment.
Prompt adaptations and response rules are given in
Appendices~\ref{app:comor_protocol} and~\ref{app:sim:prompts}.

\paragraph{Controlled comparisons.}
We add controlled comparisons in order to better localise the failures. An oracle-count condition discloses
the true $k$, and a single-diagnosis control requests one diagnosis. A matched
respiratory cohort, in which each two-diagnosis patient has a single-diagnosis
counterpart with the same opening, separates information acquisition from evidence decoding
(Section~\ref{sec:tree}; Appendix~\ref{app:evidence_matched}). A designed $k=2$
DDXPlus cohort, consulted twice with an opening from either diagnosis, tests
anchoring and symptom overlap (Section~\ref{sec:mechanisms};
Appendix~\ref{app:designed_k2}). A second designed cohort keeps each patient's record
fixed and varies only the two symptoms after the first complaint, which belong
to the first diagnosis alone or to both diagnoses (Appendix~\ref{app:sharing}).

\label{sec:scoring}
\label{sec:experiments:metrics}

We report mean per-case Jaccard and exact-set recovery:
\begin{equation}
 J=\frac{|\widehat{\mathcal D}\cap\mathcal D|}
 {|\widehat{\mathcal D}\cup\mathcal D|},
 \qquad
 \mathrm{Ex}=\mathbf{1}\{\widehat{\mathcal D}=\mathcal D\}.
 \label{eq:set_metrics}
\end{equation}

%% file: sections/results/results.tex
\section{Results}
\label{sec:results}

Clinical benchmarks usually ask for the single most likely diagnosis. A
patient, however, potentially have several conditions whose findings show \emph{merging}. We evaluate the recovery of the complete diagnosis set at every $k$, including $k=1$, without
revealing $k$, and average equally over cohorts with $k=1$ to $4$. We report this score in Table~\ref{tab:main}. For every model, interaction lowers the Jaccard point estimate, and exact-set recovery is low
even with full information, with models rarely recovering the complete
set. We note that the record does determine the set as a logistic-regression reader recovers
79 to 99\% of DDXPlus sets from it at every $k$ (Appendix~\ref{app:reader}). Error bars show 95\% bootstrap intervals over independent
units (cases, patients or families; Appendix~\ref{app:comor_analysis}).

\begin{table}[H]
\centering
\small
\caption{Pooled exact-set recovery and Jaccard under full information (F)
and interaction (I). Rows are sorted by interactive exact-set recovery on DDXPlus; shading of the interactive exact-set columns is proportional to the
value. Intervals are in Appendix~\ref{app:ci}.}
\label{tab:main}
\begin{tabular}{lcccccccc}
\toprule
 & \multicolumn{4}{c}{DDXPlus} & \multicolumn{4}{c}{ePOCT+} \\
\cmidrule(lr){2-5}\cmidrule(lr){6-9}
 & \multicolumn{2}{c}{Exact $\uparrow$} & \multicolumn{2}{c}{Jaccard $\uparrow$} & \multicolumn{2}{c}{Exact $\uparrow$} & \multicolumn{2}{c}{Jaccard $\uparrow$} \\
\cmidrule(lr){2-3}\cmidrule(lr){4-5}\cmidrule(lr){6-7}\cmidrule(lr){8-9}
Model & F & I & F & I & F & I & F & I \\
\midrule
GPT-5.6             & 0.11 & \cellcolor{cmbpurple!32}0.08 & 0.46 & 0.29 & 0.21 & \cellcolor{cmbpurple!32}0.08 & 0.57 & 0.30 \\
Gemini-3.8-Flash    & 0.26 & \cellcolor{cmbpurple!32}0.08 & 0.47 & 0.32 & 0.21 & \cellcolor{cmbpurple!40}0.10 & 0.51 & 0.39 \\
Qwen3.8-Flash       & 0.17 & \cellcolor{cmbpurple!24}0.06 & 0.46 & 0.27 & 0.18 & \cellcolor{cmbpurple!24}0.06 & 0.49 & 0.24 \\
GLM-5.3-Flash       & 0.11 & \cellcolor{cmbpurple!8}0.02 & 0.46 & 0.24 & 0.15 & \cellcolor{cmbpurple!20}0.05 & 0.52 & 0.26 \\
DeepSeek-V4.1-Flash & 0.04 & \cellcolor{cmbpurple!8}0.02 & 0.37 & 0.25 & 0.15 & \cellcolor{cmbpurple!24}0.06 & 0.53 & 0.28 \\
Gemma-4               & 0.10 & \cellcolor{cmbpurple!4}0.01 & 0.43 & 0.24 & 0.20 & \cellcolor{cmbpurple!28}0.07 & 0.55 & 0.31 \\
\bottomrule
\end{tabular}
\vspace{-2mm}
\end{table}

\input{sections/results/multiplicity.tex}

\input{sections/results/gap.tex}

%% file: sections/results/multiplicity.tex
\subsection{Does the model register multiplicity?}
\label{sec:multiplicity}

If models registered \emph{multiplicity}, that a patient may have several conditions, they
would return more diagnoses, find more correct ones (Appendix~\ref{app:theory}), and ask more questions as
$k$ grows. In practice, only the first happens (Figure~\ref{fig:mult}). Indeed, while the diagnosis set grows,
correct diagnoses rise far less than $k$ (b), and consultation length
hardly follows $k$ (c). Multiplicity changes the size of the answer, but it does not change the search behind it.

\begin{figure}[H]
\centering
\includegraphics[width=\linewidth]{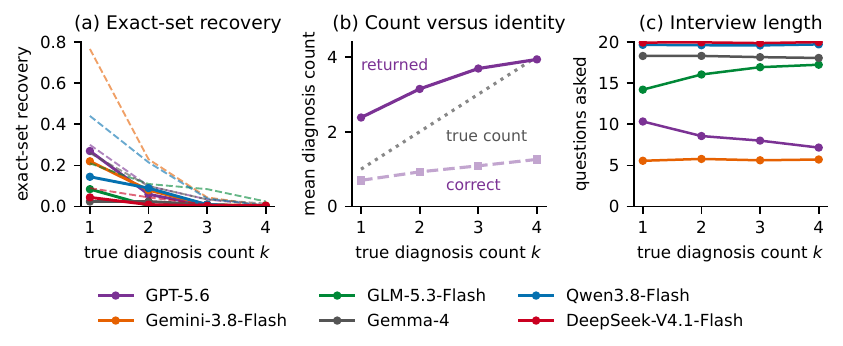}
\caption{Multiplicity on DDXPlus (six models, 200 cases per $k$). (a) Exact-set
recovery, interactive (solid) and full information (dashed).
(b) Mean returned and correct
diagnoses per interactive consultation, models pooled. (c) Questions asked.}
\label{fig:mult}
\vspace{-2mm}
\end{figure}

\begin{table}[h]
\centering
\small
\caption{Exact-set recovery (\%, $\uparrow$) by comorbidity count $k$. F is full
information, I is interactive; shading of the interactive columns is
proportional to the value. Intervals are in Appendix~\ref{app:ci}.}
\label{tab:exact_by_k}
\begin{tabular}{llcccccccc}
\toprule
 & & \multicolumn{2}{c}{$k{=}1$} & \multicolumn{2}{c}{$k{=}2$}
 & \multicolumn{2}{c}{$k{=}3$} & \multicolumn{2}{c}{$k{=}4$} \\
\cmidrule(lr){3-4}\cmidrule(lr){5-6}\cmidrule(lr){7-8}\cmidrule(lr){9-10}
Source & Model & F & I & F & I & F & I & F & I \\
\midrule
DDXPlus  & GPT-5.6             & 30 & \cellcolor{cmbpurple!40}27 & 10 & \cellcolor{cmbpurple!9}6 & 4 & 0 & 0 & 0 \\
         & Gemini-3.8-Flash    & 76 & \cellcolor{cmbpurple!33}22 & 23 & \cellcolor{cmbpurple!12}8 & 4 & \cellcolor{cmbpurple!1}1 & 0 & 0 \\
         & Qwen3.8-Flash       & 44 & \cellcolor{cmbpurple!21}14 & 22 & \cellcolor{cmbpurple!13}9 & 4 & \cellcolor{cmbpurple!1}1 & 0 & 0 \\
         & GLM-5.3-Flash       & 21 & \cellcolor{cmbpurple!12}8 & 11 & 0 & 8 & 0 & 2 & 0 \\
         & DeepSeek-V4.1-Flash & 9 & \cellcolor{cmbpurple!6}4 & 4 & \cellcolor{cmbpurple!1}1 & 1 & 0 & 0 & 0 \\
         & Gemma-4               & 26 & \cellcolor{cmbpurple!3}2 & 10 & \cellcolor{cmbpurple!3}2 & 4 & \cellcolor{cmbpurple!1}1 & 2 & 0 \\
\midrule
ePOCT+   & GPT-5.6             & 39 & \cellcolor{cmbpurple!27}18 & 24 & \cellcolor{cmbpurple!10}7 & 12 & \cellcolor{cmbpurple!6}4 & 10 & \cellcolor{cmbpurple!6}4 \\
         & Gemini-3.8-Flash    & 34 & \cellcolor{cmbpurple!33}22 & 25 & \cellcolor{cmbpurple!12}8 & 15 & \cellcolor{cmbpurple!10}7 & 10 & \cellcolor{cmbpurple!3}2 \\
         & Qwen3.8-Flash       & 34 & \cellcolor{cmbpurple!27}18 & 26 & \cellcolor{cmbpurple!7}5 & 10 & \cellcolor{cmbpurple!3}2 & 4 & 0 \\
         & GLM-5.3-Flash       & 24 & \cellcolor{cmbpurple!19}13 & 20 & \cellcolor{cmbpurple!9}6 & 10 & \cellcolor{cmbpurple!3}2 & 8 & 0 \\
         & DeepSeek-V4.1-Flash & 14 & \cellcolor{cmbpurple!19}13 & 21 & \cellcolor{cmbpurple!9}6 & 16 & \cellcolor{cmbpurple!9}6 & 11 & 0 \\
         & Gemma-4               & 32 & \cellcolor{cmbpurple!21}14 & 30 & \cellcolor{cmbpurple!12}8 & 15 & \cellcolor{cmbpurple!9}6 & 5 & 0 \\
\bottomrule
\end{tabular}
\vspace{-2mm}
\end{table}

%% file: sections/results/gap.tex
\subsection{The single-answer gap}
\label{sec:gap}
The single-answer gap is the distance between how benchmarks usually score a
model and what the task requires (Figure~\ref{fig:gap}). A single-answer
benchmark credits a consultation when the model names at least one true
condition, and this happens more often as $k$ grows. The complete set, which is the
correct target, is rarely named and falls to almost zero at $k=4$. The gap is
present already at $k=1$, where the model often finds the condition but adds
wrong ones (Appendix~\ref{app:interaction}). The same gap appears when models are asked
for a single diagnosis, and supplying the true $k$ does not close it
(Appendices~\ref{app:single} and~\ref{app:oracle}). Interaction lowers the
score at every $k$ rather than steepening its fall (Appendix~\ref{app:interaction}).

\begin{figure}[H]
\centering
\includegraphics[width=0.8\linewidth]{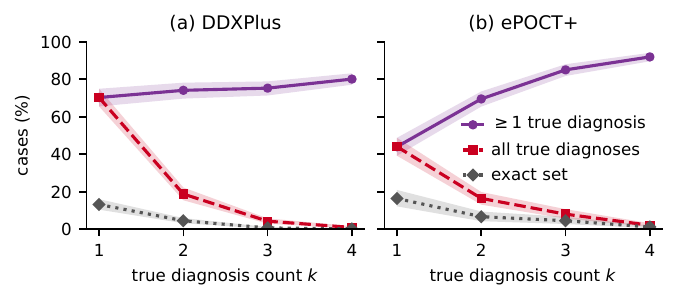}
\caption{Single-answer gap in interaction (six models, 200 cases per $k$).
Unlike the exact set, all true diagnoses may come with wrong ones. Bands are
95\% bootstrap intervals.}
\label{fig:gap}
\vspace{-2mm}
\end{figure}

%% file: sections/failure/failure.tex
\section{Failure analysis}
\label{sec:failure}

We follow the consultation step by step: seeking and integrating evidence
(Section~\ref{sec:tree}), why the second diagnosis is missed
(Section~\ref{sec:mechanisms}), and stopping (Section~\ref{sec:stopping}). The matched experiment uses GPT-5.6,
among the best models in interactive exact-set recovery on DDXPlus (Table~\ref{tab:main}).
Unless stated otherwise, we use GPT-5.6 on DDXPlus. For Sections~\ref{sec:tree} and~\ref{sec:mechanisms}, we designed controlled
experiments that isolate these effects. Their details are given in
Appendices~\ref{app:evidence_matched} and~\ref{app:designed_k2}.

\input{sections/failure/tree.tex}

\input{sections/failure/mechanisms.tex}

\input{sections/failure/stopping.tex}

%% file: sections/failure/tree.tex
\subsection{Can models acquire and use diagnostic evidence?}
\label{sec:tree}
\label{sec:integration}

A second diagnosis may be lost in two places. The model may not collect
the evidence that reveals it (a failure in evidence acquisition), or it may have sufficient evidence and still not report it (a failure in decoding). We separate the two cases on a matched cohort in which each
two-diagnosis patient has a single-diagnosis counterpart with the same
opening. As a fixed query policy we use a decision tree that asks from the
same menu, each time selecting the question that most reduces uncertainty about
the diagnosis set, and comes close to what the full record allows (cohort,
models and tree in Appendix~\ref{app:evidence_matched}).

Despite using more questions than the tree, GPT-5.6 gains less information from each (c), and under the same fixed decoder it returns lower exact-set recovery than the tree. The gap grows at $k=2$, and its selected questions also match less with those of the tree (Appendix~\ref{app:info_gain}).

The performance gap persists when GPT-5.6 is given the tree's evidence or even the full record (b), while the fixed classifier recovers most
$k=2$ sets exactly from the same full record (Appendix~\ref{app:info_gain}). We investigate this in Section~\ref{sec:mechanisms}.

\begin{figure}[t]
\centering
\includegraphics[width=\linewidth]{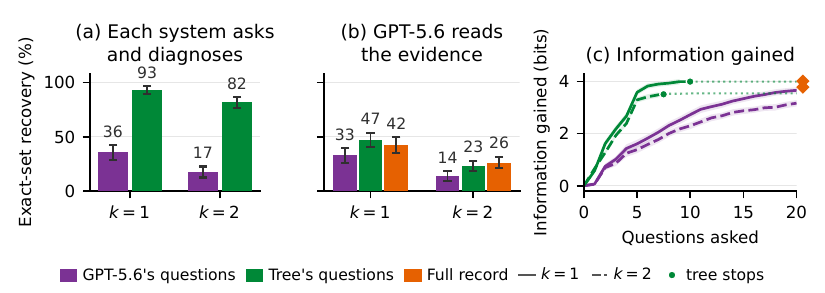}
\caption{GPT-5.6 and the reference tree on the matched respiratory cohort.
(a) Exact-set recovery when each system asks and diagnoses. (b) Exact-set
recovery when GPT-5.6 reads evidence of increasing quality. (c) Information
gained, the drop in uncertainty about the diagnosis set after each question,
scored by one fixed model for both systems (Appendix~\ref{app:info_gain}).
The dot marks where the tree stops on average, and the diamonds show the
full record.}
\label{fig:tree}
\vspace{-2mm}
\end{figure}

%% file: sections/failure/mechanisms.tex
\subsection{Why is the second diagnosis lost?}
\label{sec:mechanisms}

Clinical reasoning gives grounds on two biases that may induce the loss of a second diagnosis. With
\emph{anchoring}, the clinician holds on to the first
hypothesis~\citep{croskerry2003importance}. With \emph{satisfaction of
search}, the search stops once one diagnosis is found~\citep{berbaum1990,adamo2021}.
Both have been documented in language models on clinical
diagnosis~\citep{braitsch2026information}, and ours lose the second diagnosis
although the prompt tells them not to stop at one (Appendix~\ref{app:sim:prompts}). Both can
follow from treating diagnoses as mutually exclusive, so that a second
diagnosis must take support from the first (Proposition~\ref{prop:exclusive}). We study where the
second diagnosis is lost on designed cohorts in which the two diagnoses are equally severe
(Appendix~\ref{app:designed_k2}).

\subsubsection{Anchoring on the opening findings}
\label{sec:anchoring}
Let A denote the diagnosis whose findings open the consultation, and B the
other. GPT-5.6 anchors on A, recovering it about twice as often as B
(Figure~\ref{fig:anchoring}a), and the most common final set holds A without B
(b). Appendix~\ref{app:designed_k2} repeats this with Gemma-4.

\begin{figure}[t]
\centering
\includegraphics[width=0.85\linewidth]{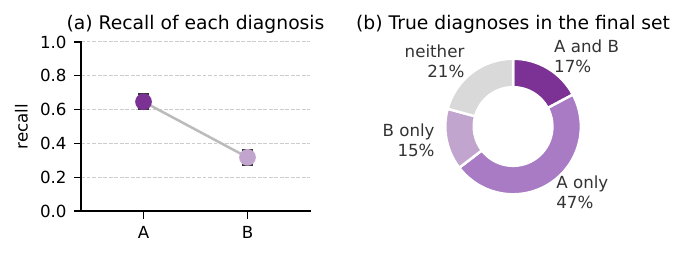}
\caption{Anchoring on the designed cohort, GPT-5.6. A is the diagnosis whose
findings open the consultation and B the other. (a) Recall of A and B.
(b) Which true diagnoses the final set contains.}
\label{fig:anchoring}
\vspace{-2mm}
\end{figure}

\subsubsection{Shared symptoms}
\label{sec:overlap}

To see how shared symptoms affect the second diagnosis, we consult the same
patients, with equally severe A and B, twice: the opening gives A's first
finding and either two symptoms that A and B share or two of A alone
(Appendix~\ref{app:sharing}). Shared symptoms raise B from 37\% to 51\% while A
barely changes (Figure~\ref{fig:app_sharing_bars}). The model finds B when a
finding in view points to it, but naming both rises only from 27\% to 32\%,
not reliably.

%% file: sections/failure/stopping.tex
\subsection{Do more questions help?}
\label{sec:stopping}
\label{sec:retention}

To investigate whether this is due to early stopping, we force the model to continue past its natural stopping point for up to
twenty questions. On 50 cases per $k$, we probe its working diagnosis set after every turn offline (Appendix~\ref{app:comor_protocol}). After twenty questions a probe holds 1.3, 1.4 and 1.9 correct diagnoses
at $k=2$, 3 and 4 (Figure~\ref{fig:retention}a). More questions mostly add wrong
diagnoses (b), so Jaccard does not rise after the natural stop (c). The other
open models behave the same (Appendix~\ref{app:compliance}). %

\begin{figure}[H]
\centering
\includegraphics[width=\linewidth]{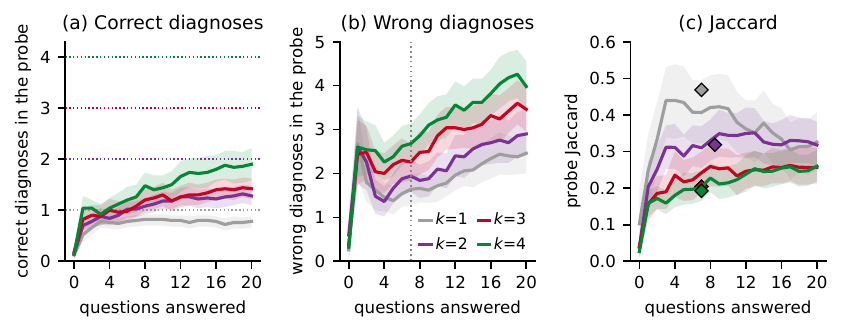}
\caption{Single probes by question turn (DDXPlus, GPT-5.6, 50 cases per
$k$). (a) Correct and (b) wrong diagnoses. (c) Jaccard; diamonds mark the
committed answer at the median natural stop. Dotted lines mark $k$ in (a) and
the median stop in (b).}
\label{fig:retention}
\vspace{-2mm}
\end{figure}

Together, these failures describe a single-hypothesis tracker, formalised in
Appendix~\ref{app:theory}. The
model follows the opening, misses a diagnosis unless its current set already point to it, and forcing further
questions do not complete the set.

%% file: sections/discussion.tex
\section{Discussion and conclusion}
\label{sec:discussion}
\label{sec:conclusion}

We introduced \CLIMB, a benchmark for interactive diagnosis of patients
with several co-occurring conditions. Models often name one of a patient's
conditions but rarely recover the complete set, and supplying the number of
conditions or even the full record does not close the gap. In the interactive setting, we observe single-hypothesis tracking behaviour. The opening presentation decides which diagnosis they
pursue, their questions stay on it, and
further questions add mostly wrong diagnoses (Appendix~\ref{app:theory}). This
is a concerning profile for clinical use, where multimorbidity is common, and
evaluations that score a single most likely diagnosis hide it. Future work
should improve evidence gathering and the tracking of several hypotheses at
once.

%% file: appendix/extended_related_work.tex
\section{Extended Related Work}
\label{app:extended_related_work}

\paragraph{From static questions to interactive consultations.}
Medical language models are often assessed on self-contained exam questions,
including MedQA, MedMCQA, and PubMedQA \citep{jin2021medqa,pal2022medmcqa,
jin2019pubmedqa}. These establish whether a model can use the full evidence set
supplied upfront. They do
not assess the model in deciding what information is missing, or in
maintaining a diagnostic state across a conversation. Converting a vignette into a dialogue can change the task even if its clinical
content is held fixed. CRAFT-MD evaluates such conversational reasoning across
medical specialties, and reports substantially different behavior from
vignette-based assessment \citep{johri2025evaluation}. AgentClinic broadens
this setting to simulated patients, measurements, multimodal tools, and
environmental biases \citep{schmidgall2024agentclinic}. These studies motivate
evaluating the whole consultation rather than treating static diagnostic
accuracy as a proxy for information gathering.

\paragraph{Question asking and patient simulation.}
MediQ lets a model abstain from answering and
ask a follow-up question when it lacks sufficient evidence
\citep{li2024mediq}. Directly prompting a model to ask
questions may reduce diagnostic performance, and interactive accuracy generally stays below that
of a model given complete information upfront. Other factors may influence LLM-mediated clinical interactions: CRAFT-MD uses language-model agents both to play the patient and to score the
diagnosis, and AgentClinic reports that doctor accuracy changes with the model
that plays the patient and falls when the patient is biased. Accordingly, \CLIMB{} fixes the generated case,
candidate vocabulary, and scorer across its full-information and interactive
conditions. The comparison asks how much recovery changes when a doctor must
elicit the same retained findings rather than receive them upfront. Our
free-text patient interface follows this line of work; the menu condition
instead returns deterministic record-based answers for analyses that require
an unambiguous evidence trace.

\paragraph{Concurrent diagnoses in interactive benchmarks.}
The premise that a consultation can involve more than one disease is not new.
Dr.~Assistant evaluates multi-turn inquiry on cases stratified from one to six
diseases and treats the multi-disease strata as a multi-label task
\citep{guo2026drassistant}. Its primary evaluation uses ICD-code recall, with
precision reported alongside it. LingxiDiagBench evaluates both static and
dynamic Chinese psychiatric consultations; its four-way setting includes pure
depression, pure anxiety, mixed depression--anxiety, and other psychiatric
conditions \citep{xu2026lingxidiagbench}. The mixed category tests recognition
of a specific comorbidity pattern, rather than recovery of an arbitrary set of
simultaneously present diagnoses.

Our benchmark asks a different question. A patient may have several
conditions whose findings are merged into one presentation, so the task is
neither to name the most likely diagnosis nor to rank a differential for a
single cause, but to return the set of conditions that are actually present.
We therefore ask for the complete set in every case, including $k=1$, without
revealing how many diagnoses there are, and score it as a whole: a reliable
answer must include every present condition, above all the dangerous ones,
and exclude those that are absent. Averaging this score over cohorts with
$k=1$ to $4$ gives the measure we argue a benchmark for multimorbid patients
should report, rather than recall of an arbitrary list or recognition of one
predefined comorbidity pattern. We systematically observe the effect of $k$ and beyond measuring the decline with $k$, we use this control to
locate where the second diagnosis is lost, separating evidence acquisition
from evidence use and testing anchoring and symptom overlap in designed
cohorts.

\paragraph{What is evaluated as a ``set.''}
The important distinction is between a differential and a concurrent target
set. A differential diagnosis contains candidate explanations, often ordered
by plausibility; it need not assert that every candidate is true. In contrast,
our target $\mathcal D$ contains every diagnosis used to construct the case,
and all members are simultaneously scored as present. DDXPlus illustrates the
difference: each synthetic record includes one ground-truth pathology together
with a differential diagnosis, symptoms, and antecedents
\citep{fansitchango2022ddxplus}. We use its individual records as source
components, not its differential list as a multimorbidity label. This permits
us to form an unordered concurrent set by sampling diagnoses and merging
compatible records under age, sex, and finding-consistency constraints.

This design also determines the metrics. Multi-label learning distinguishes
per-label objectives from subset-exact evaluation, in which an output is
correct only when its complete label set is correct
\citep{dembczynski2012label}. Exact match is deliberately stringent, so we
report it together with Jaccard, precision, recall, and predicted cardinality.
These measures separate partial recovery from returning unsupported labels and
from correctly registering the number of diagnoses. Recent work on language
models in multi-label prediction likewise finds that the output formulation can
affect how many labels are emitted \citep{ma2025multilabel}. Our contribution
is not a new multi-label objective; it is to evaluate these set-level outcomes
when the evidence must first be acquired during a consultation.

\paragraph{Controlling case composition.}
Clinical coding, phenotyping, and comorbidity extraction also predict multiple
labels, but begin from a fixed record whose evidence is already available
\citep{dengcomorbidity2026}. They are therefore useful comparators for set
prediction but do not isolate question selection. Real multimorbid encounter
records would offer strong clinical realism, yet their condition labels,
evidence provenance, and privacy constraints make matched interventions
difficult. \CLIMB{} instead uses two source-grounded constructions. The
ePOCT+ branch composes compatible paths from a pediatric clinical decision
algorithm \citep{tan2023epoct}; the DDXPlus branch merges compatible source
records. In both, the accepted case retains a target set and component-level
support. This permits controlled changes to diagnosis count and information
condition while keeping the evaluation target explicit. The construction is a
synthetic benchmark, not a claim that its generated distribution is a clinical
cohort; Appendix~\ref{app:generation} documents the source resources,
constraints, and generated-population profiles.

\paragraph{Diagnostic error and classical decision support.}
Clinical diagnostic-error research provides useful concepts without identifying
the mechanism of an LLM failure. Premature closure is the failure to continue
considering reasonable alternatives after an initial diagnosis, and was the
most common cognitive factor in a retrospective study of internal-medicine
diagnostic errors \citep{graber2005diagnostic}. Anchoring and premature closure
therefore motivate analyses of whether an early presentation receives more
attention, whether diagnoses appear and disappear from successive outputs, and
whether additional questioning changes the final set. They do not license an
inference from any one of these observables to a model's internal state.

Earlier clinical decision-support systems already represented multiple diseases
and their shared findings. INTERNIST-1 was designed as a general internal
medicine diagnostic consultant \citep{miller1982internist}, while QMR-DT
reformulated its knowledge base as a probabilistic model over diseases and
findings \citep{shwe1991probabilistic}. Joint disease representation and
interactive diagnosis are therefore longstanding ideas. There is also an ongoing research direction to merge these classical system with LLMs \citep{kesmen2026mobayesmodularbayesianframework}.

\CLIMB{} applies
them to contemporary language-model consultations with a reproducible,
component-preserving synthetic generator; full-set scoring; and matched
comparisons of evidence availability, diagnosis trajectories, and stopping.

%% file: appendix/patient_generation/patient_generation.tex
\section{Patient Generation}
\label{app:generation}
\label{app:comor_details}

\input{appendix/patient_generation/dataset_overview.tex}

\input{appendix/patient_generation/comor_epoct.tex}

\input{appendix/patient_generation/comor_ddx.tex}

\input{appendix/patient_generation/comor_masks.tex}

\input{appendix/patient_generation/comor_background.tex}

%% file: appendix/patient_generation/dataset_overview.tex
\subsection{Source Data and Record Representation}
\label{app:comor_sources}

The two sources provide different building blocks: ePOCT+ supplies
clinical decision rules and cohort-derived diagnosis frequencies;
DDXPlus supplies synthetic patient records with one designated target
pathology. The populations constructed from them contain sets of
concurrent target diagnoses. Source-resource counts below describe
the local files used for construction; generated-population statistics
are reported separately in Appendix~\ref{app:comor_population}.

\subsubsection{ePOCT+: Pediatric Decision Rules and Cohort Statistics}
\label{app:source_epoct}

ePOCT+ is a clinical decision-support algorithm for assessing sick
children under 15 years in primary care. It combines history, symptoms,
examination findings, and selected point-of-care tests to guide diagnosis
and management; the associated clinical trial evaluated its use in
Tanzania \citep{tan2023epoct}. In our construction, the algorithm supplies
the findings associated with a diagnosis, while recorded cohort diagnoses
supply its sampling frequency and pairwise associations.

The local algorithm export contains 617 nodes, 132 diagnosis branches,
and 235 final-diagnosis IDs representing 216 distinct normalized label
strings. These counts refer to different objects: several IDs can share
a label, and branches can reuse clinical variables. Nodes include
symptoms, signs, examinations, tests, routing conditions, and management
questions. The generator removes treatment/referral conditions and
withholds routing/calculation nodes from the emitted findings.

The stored cohort aggregate has a denominator of 36,779 records,
166 diagnosis-label strings, and 1,959 observed unordered label pairs.
These are the contents of the local aggregate, rather than the clinical
trial's enrollment or a count of unique children. Matching labels to
the graph and applying the label mask leaves 142 proposal labels and
1,363 observed pairs between them. The finite accepted population
contains fewer labels, as reported below. Figure~\ref{fig:comor_component_example}
shows how two selected components produce one record.

\subsubsection{DDXPlus: Synthetic Records for Differential Diagnosis}
\label{app:source_ddx}

DDXPlus is a public synthetic dataset developed for automatic symptom
collection and diagnosis \citep{fansitchango2022ddxplus}. Its original
construction combines a medical knowledge base, demographic simulation,
and a rule-based diagnostic system. The released condition set covers
49 pathologies associated with presentations involving cough, sore
throat, or breathing difficulties. It includes both children and adults.
Each record has age, recorded sex, an initial symptom, symptom and
antecedent evidence, one target pathology, and a differential diagnosis.
The differential lists alternative explanations; it is not a set of
confirmed concurrent diseases. Our target sets are built by combining
records indexed by their target pathology.

The local copy contains 1,292,579 records: 1,025,602 training, 132,448
validation, and 134,529 test records. All three splits contain all 49
target pathologies and span recorded ages 0--109 years. Construction
uses the training split for prevalence estimates and source records;
the cached sampling pools contain 71,479 records, capped at 1,500 per
pathology. These source splits are distinct from our subsequently
generated benchmark population and its evaluation subset.

The evidence schema contains 223 variables: 110 symptoms and 113
antecedents, with the answer types in Table~\ref{tab:comor_ddx_schema}.
Antecedents describe medical history and risk factors. Binary variables
encode presence; categorical variables encode one answer, including
ordered scales; multi-select variables can carry several answers.
For example, the fever question is a binary variable. Evidence IDs and
value tokens are rendered into question--answer statements when records
are merged. A multi-select variable contributes one rendered statement,
even if it contains several selected values.

\input{appendix/patient_generation/data_profile/ddx_schema_table.tex}

\subsection{Generated Populations and Evaluation Subsets}
\label{app:comor_population}

\paragraph{Populations and evaluated subsets.}
Each stored population contains 800 cases, with 200 at each
$k\in\{1,2,3,4\}$ (Table~\ref{tab:comor_population_strata}). The ePOCT+
file contains 96 observed diagnosis labels and 418 distinct target sets;
DDXPlus contains 47 of its 49 source labels and 595 target sets. The main
evaluation uses every stored record, 200 at each $k$ and 800 per source
(Table~\ref{tab:comor_loader_profile}), for every condition and every model.
The oracle-count and single-diagnosis controls use the first 50 records at
each $k$.

\input{appendix/patient_generation/data_profile/population_table.tex}
\input{appendix/patient_generation/data_profile/loader_table.tex}

\paragraph{Findings and demographics.}
A signal finding is one retained rendered statement after excluding
the separately tagged background additions; it can be a symptom,
antecedent, test, or other retained rule variable. Figure~\ref{fig:comor_finding_distributions}
shows its distribution by $k$. The ePOCT+ medians are
3, 6, 10, and 14 findings; the DDXPlus medians are 15, 24, 32, and 38.
These counts describe record length, not the number of positive symptoms
or the number of questions needed to identify the targets.

Across all 800 cases, ePOCT+ ages have median 7.8 years
(interquartile range 3.8--11.5), with 49.5\% recorded female; DDXPlus
has median 44 years (28--61), with 51.2\% recorded female. These are
generated demographics: ePOCT+ ages are sampled from compatible
windows, while DDXPlus uses a compatible anchor age. They are not age
distributions measured in the ePOCT+ clinical cohort.

\begin{figure}[htbp]
\centering
\includegraphics[width=\linewidth]{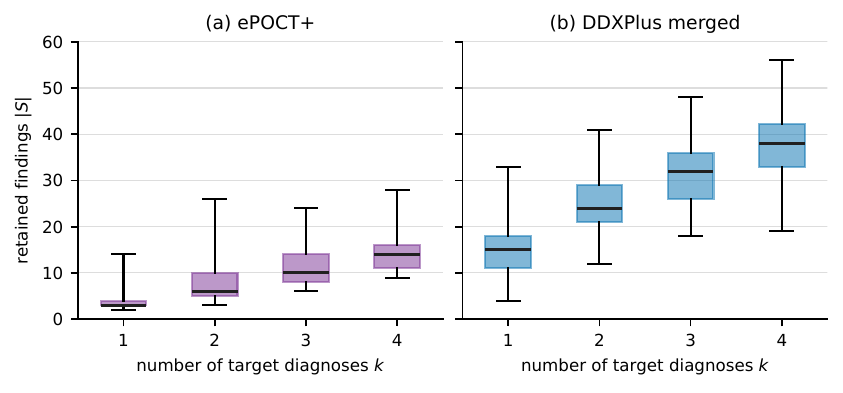}
\caption{Finding counts in the ePOCT+ and merged DDXPlus
populations, with 200 cases per box. Boxes show the interquartile range,
central lines the median, and whiskers the full observed range.
Separately injected background complaints are excluded.}
\label{fig:comor_finding_distributions}
\end{figure}

\paragraph{Evidence support.}
Every sampled ePOCT+ label has a retained path under the rule of
Appendix~\ref{app:comor_epoct}, and a proposal is accepted only if each of
its diagnoses keeps a visible finding that no other component contributes.
Each case stores these components. In all 800 cases their visible findings
reproduce the stored signal, so every target is supported by evidence in the
record. Figure~\ref{fig:comor_component_example} shows one case. The DDXPlus
generator applies the same rule: a case is rejected if any diagnosis keeps no
finding of its own after merging (Algorithm~\ref{alg:comor_ddx}).

\paragraph{The record determines the set.}
\label{app:reader}
To check that a target set can be recovered from its record, we train a
reference reader on the same input the models receive under full information,
the opening and the retained findings. It is a one-vs-rest logistic regression
over binary indicators of the findings and the opening (chief complaint, sex,
age band), trained on cases drawn from the same generators with a different
seed after removing exact duplicates of evaluated records. It returns every
label with probability at least $0.5$, or the most probable label if none
passes; this rule was chosen on a further 7{,}201 generated cases, and the
evaluated records are only scored. Given the true $k$, it returns the $k$ most
probable labels. On DDXPlus, trained on 79{,}877 cases, it recovers most sets
at every $k$, far above every model (Table~\ref{tab:reader}). Accuracy has
saturated with training size: at $k=4$ it is 69\%, 78\% and 79\% with 5{,}000,
20{,}000 and 80{,}000 training cases. ePOCT+ records almost identify their
sets, since 43\% of freshly generated cases duplicate an evaluated record, and
the reader recovers 99 to 100\% at every $k$. The reader is a learned
reference on generator output rather than an optimal questioning policy; it
shows that the full record determines most target sets, so low exact-set
recovery is not a property of the task.

\begin{table}[h]
\centering
\small
\caption{Exact-set recovery (\%) from the full DDXPlus record, 200 cases per
$k$: reference reader without and with the true $k$, the best of the six
models, and GPT-5.6. Brackets are 95\% case-bootstrap intervals.}
\label{tab:reader}
\begin{tabular}{lcccc}
\toprule
$k$ & Reader & Reader, $k$ given & Best model & GPT-5.6 \\
\midrule
1 & 99.5 [98.5, 100] & 100 & 76.5 & 30.0 \\
2 & 91.0 [87.0, 94.5] & 96.5 & 23.0 & 10.0 \\
3 & 79.5 [74.5, 85.5] & 91.0 & 8.5 & 3.5 \\
4 & 79.0 [73.0, 84.5] & 89.5 & 2.5 & 0.5 \\
\bottomrule
\end{tabular}
\end{table}

\input{appendix/patient_generation/data_profile/component_example.tex}

\paragraph{Reproducibility.}
The accompanying script \texttt{tools/describe\_benchmark\_data.py}
computes these summaries from the local files and exports per-diagnosis
frequencies, population summaries, the figure, and source-file hashes.
Source split sizes are checked against the actual Arrow records. Quartiles
use linear interpolation. \texttt{tools/epoct\_v2\_support.py} checks the
stored components of every ePOCT+ case against its signal and draws the
worked example.

%% file: appendix/patient_generation/data_profile/ddx_schema_table.tex
\begin{table}[htbp]
\centering\small
\caption{Evidence-variable schema in the local DDXPlus release. Antecedents encode medical history and risk factors. Counts refer to variables, not answer tokens.}
\label{tab:comor_ddx_schema}
\begin{tabular}{@{}lrrr@{}}
\toprule
Answer type & Symptoms & Antecedents & Total \\
\midrule
Binary & 96 & 112 & 208 \\
Categorical & 9 & 1 & 10 \\
Multi-select & 5 & 0 & 5 \\
\midrule
Total & 110 & 113 & 223 \\
\bottomrule
\end{tabular}
\end{table}

%% file: appendix/patient_generation/data_profile/population_table.tex
\begin{table}[htbp]
\centering\small
\caption{Stored populations. Finding counts exclude added background and show median [first, third quartile]. Diagnosis-set counts ignore ordering. Each stratum has 200 cases.}
\label{tab:comor_population_strata}
\begin{tabular}{@{}llrrrr@{}}
\toprule
Population & $k$ & Cases & Unique labels & Distinct sets & Signal findings \\
\midrule
ePOCT+ & 1 & 200 & 33 & 33 & 3 [3, 4] \\
 & 2 & 200 & 47 & 92 & 6 [5, 10] \\
 & 3 & 200 & 68 & 134 & 10 [8, 14] \\
 & 4 & 200 & 72 & 159 & 14 [11, 16] \\
\midrule
DDXPlus merged & 1 & 200 & 44 & 44 & 15 [11, 18] \\
 & 2 & 200 & 46 & 163 & 24 [21, 29] \\
 & 3 & 200 & 46 & 191 & 32 [26, 36] \\
 & 4 & 200 & 47 & 197 & 38 [33, 42.2] \\
\bottomrule
\end{tabular}
\end{table}

%% file: appendix/patient_generation/data_profile/loader_table.tex
\begin{table}[htbp]
\centering\small
\caption{Evaluated cases: all 200 stored cases at each $k$ of each population. Label counts describe this subset, not the candidate vocabulary. Ages are generated ages in years; F is the recorded female category.}
\label{tab:comor_loader_profile}
\begin{tabular}{@{}lrrrr@{}}
\toprule
Population & Cases & Unique labels & Age, median [Q1, Q3] & F (\%) \\
\midrule
ePOCT+ & 800 & 96 & 7.8 [3.8, 11.5] & 49.5 \\
DDXPlus merged & 800 & 47 & 44 [28, 61] & 51.2 \\
\bottomrule
\end{tabular}
\end{table}

%% file: appendix/patient_generation/data_profile/component_example.tex
\begin{figure}[htbp]
\centering
\resizebox{\linewidth}{!}{%
\begin{tikzpicture}[
 component/.style={draw=black!55,fill=black!3,rounded corners=2pt,
 align=left,text width=6.1cm,inner sep=7pt,font=\small},
 merged/.style={draw=black!55,align=center,text width=11cm,inner sep=6pt,font=\small}]
\node[component,anchor=north] (c0) at (-3.5,0) {\textbf{Bacterial acute pharyngitis}\\ Cape Town CDR: Yes\\ \textbf{Runny or blocked nose: Yes}\\ Sore throat: Yes\\ Tonsillar swelling: Present};
\node[component,anchor=north] (c1) at (3.5,0) {\textbf{Common Cold}\\ Cough: Yes\\ Difficulty breathing: Yes\\ \textbf{Runny or blocked nose: Yes}};
\node[merged,anchor=north] (signal) at (0,-2.5) {\textbf{Merged signal: 6 distinct findings}\\
Shared findings (bold) appear once in the record.};
\draw[->] (c0.south) -- (signal.north west);
\draw[->] (c1.south) -- (signal.north east);
\node[merged,anchor=north] (record) at (0,-3.9) {\textbf{Emitted case: 2 target diagnoses, 6 signal findings, 2 background complaints}\\
Background: night-time urination; abdominal bloating.};
\draw[->] (signal.south) -- (record.north);
\end{tikzpicture}}
\caption{A stored ePOCT+ case (\texttt{epoct\_k2\_81}). Each diagnosis keeps at least one
finding of its own; shared findings are retained once. Hidden routing and age
conditions are omitted.}
\label{fig:comor_component_example}
\end{figure}
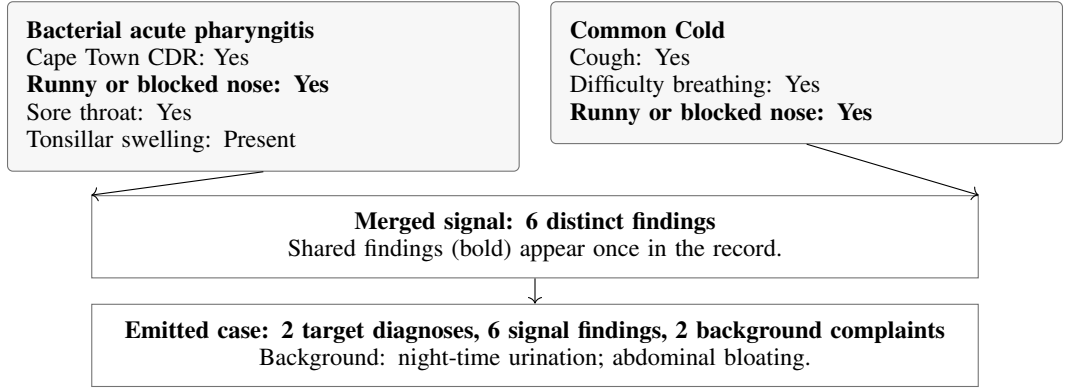

%% file: appendix/patient_generation/comor_epoct.tex
\subsection{ePOCT+ Statistics, Paths, and Rejection}
\label{app:comor_epoct}
\label{app:data}

The cohort table records accepted final diagnoses for each consultation.
We count consultations $n$, diagnoses $c_d$, and unordered pairs
$c_{dd'}$ to estimate $\widehat P(d)=c_d/n$ and
$\widehat P(d,d')=c_{dd'}/n$. Label matching maps cohort entries to
final-diagnosis IDs in the source graph; the label mask in
Appendix~\ref{app:comor_masks} removes screened entries. For observed
pairs, the sampling score is
$\log\widehat P(d)+\sum_{d'\in R}\log\ell(d,d')$.
An unobserved pair contributes $-10^9$. This is a score penalty,
not an explicit pair-exclusion test: if every remaining candidate
receives a penalty, the routine still normalizes the available scores.

Algorithm~\ref{alg:comor_epoct} summarizes the sampling loop that
produced the evaluated population. For a proposed set of $k$ diagnoses,
construction proceeds as follows.
\begin{enumerate}
 \item Enumerate root-to-diagnosis paths in each diagnosis branch and
 collect their conditions in path order, including conditions on the
 final diagnosis. Drop treatment and referral conditions.
 \item Treat several answers listed for one variable as alternatives and
 intersect them along the path. Retain a path if every variable keeps at
 least one answer and the path has a visible finding.
 \item Choose a retained path uniformly and one allowed answer per
 variable. A diagnosis with no retained path rejects the proposal.
 \item Merge assignments. Reject conflicting answers for a shared
 variable, a pair listed in either diagnosis's exclusion rules, or an
 empty intersection of branch age windows. Require every diagnosis to
 keep a visible finding that no other component contributes, and at
 least two visible findings after merging; complaint-category and
 background-calculation nodes do not count as visible findings.
 \item On rejection, draw a new target set of the same cardinality;
 a diagnosis is never dropped to reduce $k$.
\end{enumerate}

For an accepted signal, age is sampled uniformly from the intersected
window, bounded by birth and 15 years, and sex is sampled uniformly
from the two recorded categories, and the case receives an opening
presentation. The
rejection checks establish consistency of the extracted components;
they do not establish clinical label completeness.

\paragraph{Evidence support.}
Each stored case keeps the component of every diagnosis. In all 800
cases, the visible findings of the components reproduce the stored
signal, and every diagnosis has at least one finding of its own
(Appendix~\ref{app:comor_population}).

%% file: appendix/patient_generation/comor_ddx.tex
\subsection{DDXPlus Associations, Masks, and Record Merging}
\label{app:comor_ddx}
\label{app:ddx_generation}

\paragraph{Prevalence and co-occurrence ratios.}
The current generator estimates pathology prevalence from the DDXPlus
training split and caches up to 1,500 source records per pathology
using reservoir sampling. An offline elicitation asks an LLM for each
pair's co-occurrence multiplier relative to independence, separately
from how common the diseases are. The stored ratios aggregate
two elicitation rollouts by their geometric mean, equivalently the
arithmetic mean of log ratios. They are not direct observations
of concurrent diagnoses in DDXPlus.
Appendix~\ref{app:comor_elicitation} gives the association and complaint-phrase
templates used to prepare these stored resources.

The associated probability-construction script converts a multiplier
into a pairwise joint probability and clips it to its Fr\'echet bounds:
\begin{equation}
 \begin{aligned}
 \widetilde P(d,d')&=\operatorname{clip}\!\left(
 \ell(d,d')\widehat P(d)\widehat P(d'),L_{dd'},U_{dd'}\right),\\
 L_{dd'}&=\max\{0,\widehat P(d)+\widehat P(d')-1\},\qquad
 U_{dd'}=\min\{\widehat P(d),\widehat P(d')\}.
 \end{aligned}
 \label{eq:comor_clipped_joint}
\end{equation}
Clipping bounds each pairwise probability; it does not validate a
higher-order joint distribution. The current sequential generator
uses the stored multipliers directly in Equation~\ref{eq:sampler},
rather than sampling from this clipped joint matrix. The matrix belongs
to the earlier copula-based construction, which used marginals from
all source splits; it is not an additional stage of the present sampler.

\paragraph{Compatibility mask.}
A separate elicitation assigns each pair an infeasible, unlikely, or
plausible verdict (Appendix~\ref{app:comor_masks}). Only infeasible pairs
are hard exclusions. At each sampling step, a candidate is available
only if it forms no excluded pair with the selected set and preserves
a nonempty common age window. If no candidate is available, sampling
restarts. This check is distinct from association weighting: a low
multiplier reduces probability, whereas an exclusion forbids the pair.

\paragraph{Realization and rejection.}
Algorithm~\ref{alg:comor_ddx} gives the diagnosis and record sampling loops.
Each pathology's age window is the empirical 5th--95th percentile
interval in its cached record pool. For a sampled set, choose a sex
represented in every pool and an integer anchor age in the common
window. Draw one record per diagnosis of that sex within ten years
of the anchor. Merge binary evidence by union; merge multi-select
values by union, integer-valued ordered categorical answers by maximum,
and reject a nominal categorical variable with multiple distinct values.
The implementation identifies ordered variables from integer-valued
answer schemas.

Record selection is retried up to 60 times. Failure to find a compatible
realization causes the caller to resample the diagnosis set. Accepted
signals must contain at least two rendered findings. The final opening
uses the anchor age and the initial evidence of the target pathology
with the highest recorded severity. Stored phrase mappings turn that
evidence into a complaint. Neither association elicitation nor masking
is called during patient generation.

%% file: appendix/patient_generation/comor_masks.tex
\subsection{Offline Masking Prompts}
\label{app:comor_masks}

Both masking scripts send batches of up to 40 numbered entries with
the system message: ``You are a careful clinician. Output ONLY JSON.''
The default model in the scripts is \texttt{openai/gpt-5.6}, overridable
through the elicitation-model setting. The following are the user
templates, with batch-specific lists and final indices represented
by placeholders.

\paragraph{ePOCT+ label mask.}
Every returned verdict other than \texttt{ok} is written to the label
mask and excluded from diagnosis sampling.
\begin{quote}\small\raggedright
These labels come from the final-diagnosis list of a pediatric clinical
decision tree. For each numbered label, judge whether it is a real
disease / clinical diagnosis, or a nonsense entry for a disease list.

Answer `ok' for a real disease, syndrome or clinical classification
(e.g. `Severe pneumonia' is ok). Otherwise answer with one word for
why it is NOT a disease:

status = patient state (HIV status, exposure, weight-for-age, mother's status)\\
test\_result = outcome of a test (negative HIV test, negative pregnancy test)\\
follow\_up = follow-up visit outcome (improved / not improved / resolved X)\\
symptom = a symptom or sign, not a diagnosis (tooth pain, loss of appetite)\\
admin = visit type or care gap (follow-up consultation, screening, vaccination)\\
bucket = vague severity/danger bucket (critical illness, danger signs)

[Numbered labels]

Return ONLY a JSON object mapping each number to its verdict, e.g.
\texttt{\{"0":"ok","1":"status"\}}. Include every number 0..[last].
\end{quote}

\paragraph{DDXPlus pairwise mask.}
Verdicts 0 and 1 are recorded in the mask table; only verdict 0 is
enforced as a hard exclusion by the generator.
\begin{quote}\small\raggedright
For each numbered pair of diseases, judge whether one patient can have
BOTH at the same time.

0 = infeasible: cannot coexist (mutually exclusive by definition,
contradictory states of the same disease, or one rules out the other).\\
1 = possible but VERY unlikely to co-occur in practice (e.g. typical-age
mismatch, near-disjoint patient populations).\\
2 = plausible co-occurrence (even if uncommon).

[Numbered disease pairs]

Return ONLY a JSON object mapping each number to 0, 1 or 2, e.g.
\texttt{\{"0":2,"1":0\}}. Include every number 0..[last].
\end{quote}

These screens propose the masks, and our clinicians verified both, the
ePOCT+ label mask and the DDXPlus exclusions. The masks restrict which labels and pairs are sampled; they do
not replace a clinical review of every generated case.

%% file: appendix/patient_generation/comor_background.tex
\subsection{Background Sampling}
\label{app:comor_background}
\label{app:background}

Let $\mathcal B_s$ be the retained symptom vocabulary for recorded sex
$s$, with survey weights $w_s(b)$. For ePOCT+, retain only entries
marked as applicable to children. Write the count as the shifted
binomial variable
\begin{equation}
 N=1+Z,\qquad Z\sim\operatorname{Binomial}(1,\tfrac12).
 \label{eq:comor_background_count}
\end{equation}
Thus $\Pr(N=1)=\Pr(N=2)=1/2$, $\mathbb E[N]=3/2$, and
$\operatorname{Var}(N)=1/4$. This is distributionally identical to
the implemented integer draw from one through two; no generator
change is needed for this parameterization. Conditional on $N$, sample sequentially without
replacement. If $B_{j-1}$ contains
previously selected complaints,
\begin{equation}
 \Pr(b_j=b\mid s,B_{j-1})=
 \frac{w_s(b)}{\sum_{u\in\mathcal B_s\setminus B_{j-1}}w_s(u)},
 \qquad b\in\mathcal B_s\setminus B_{j-1}.
\end{equation}
If the vocabulary is smaller than the requested count, the routine
uses its available entries. The current vocabulary has sufficient
entries for the requested counts. Each sampled complaint is rendered
as an atomic reported finding and recorded in a separate background
list. The survey concerns adults; pediatric filtering is a construction
choice rather than an estimate of pediatric prevalence. Added complaints
are not subjected to a further diagnosis-preservation check in the
current code.

The evaluation loader presents the complete record $x=S\cup B$ in every
condition; the same record feeds the full-information prompt and the
patient simulator.

\input{appendix/patient_generation/algorithms.tex}
\input{appendix/patient_generation/elicitation_prompts.tex}

%% file: appendix/patient_generation/algorithms.tex
\subsection{Construction Algorithms}
\label{app:signal_algo}

The masks, source statistics, and record pools are prepared once. The
algorithms below describe case generation after this preparation;
they make no LLM calls. $\bot$ denotes a rejected proposal.
\textsc{DrawBackground} uses Equation~\ref{eq:comor_background_count}
and the weighted sampling rule in Appendix~\ref{app:comor_background}.

\begin{algorithm}[H]
\caption{ePOCT+ generation of $n_k$ cases at cardinality $k$}
\label{alg:comor_epoct}
\begin{algorithmic}[1]
\Require Graph $\mathcal T$, screened labels, prevalence and association tables
\State $C\gets[\,]$
\While{$|C|<n_k$}
 \State Draw $\mathcal D$ of size $k$ with Equation~\ref{eq:sampler}
 \State $U\gets\emptyset$; $W\gets[0,15\text{ years})$; $\mathrm{valid}\gets\mathrm{true}$
 \For{$d\in\mathcal D$}
  \State Enumerate root-to-diagnosis paths; remove treatment/referral conditions
  \State Retain paths on which every variable keeps an allowed answer and a finding is visible
  \State If none exists, $\mathrm{valid}\gets\mathrm{false}$; \textbf{break}
  \State Draw a retained path uniformly and one allowed answer per variable
  \State Merge its assignments into $U$; intersect $W$ with the diagnosis age window
  \If{answers conflict or $W$ is empty}
   \State $\mathrm{valid}\gets\mathrm{false}$; \textbf{break}
  \EndIf
 \EndFor
 \State $S\gets$ visible findings in $U$
 \If{not valid, an exclusion holds, a diagnosis has no own visible finding, or $|S|<2$}
  \State \textbf{continue} \Comment{Resample all $k$ diagnoses}
 \EndIf
 \State Sample age in $W$ and recorded sex $s$; form the opening presentation
 \State $B\gets\Call{DrawBackground}{s,\mathrm{pediatric}=\mathrm{true}}$
 \State Append $(\mathcal D,x=S\cup B,B,\text{opening})$ to $C$
\EndWhile
\State \Return $C$
\end{algorithmic}
\end{algorithm}

\begin{algorithm}[H]
\caption{DDXPlus generation of $n_k$ cases at cardinality $k$}
\label{alg:comor_ddx}
\begin{algorithmic}[1]
\Require Pathology pools, prevalence and multiplier tables, infeasible-pair mask, age windows
\State $C\gets[\,]$
\While{$|C|<n_k$}
 \State $R\gets\emptyset$; $W\gets[0,200]$
 \For{$j=1,\ldots,k$}
  \State $A\gets$ unselected diagnoses with no masked pair against $R$ and nonempty age intersection
  \If{$A=\emptyset$} \State $R\gets\bot$; \textbf{break} \EndIf
  \State Draw $d$ with Equation~\ref{eq:sampler} normalized over $A$
  \State $R\gets R\cup\{d\}$; intersect $W$ with the age window for $d$
 \EndFor
 \If{$R=\bot$} \State \textbf{continue} \EndIf
 \If{no sex is represented in every selected pool} \State \textbf{continue} \EndIf
 \State $S\gets\bot$
 \For{up to 60 realization attempts}
  \State Choose a sex $s$ present in every selected pool and an anchor age in $W$
  \State Choose one record per diagnosis of that sex within ten years of the anchor
  \If{a required record is unavailable} \State \textbf{continue} \EndIf
  \State Merge binary/multi-select values by union and ordered values by maximum
  \If{nominal values conflict} \State \textbf{continue} \EndIf
  \If{some diagnosis keeps no finding of its own} \State \textbf{continue} \EndIf
  \State $S\gets$ rendered merged findings; \textbf{break}
 \EndFor
 \If{$S=\bot$ or $|S|<2$} \State \textbf{continue} \EndIf
 \State Form the opening using the anchor and highest-severity component's initial complaint
 \State $B\gets\Call{DrawBackground}{s,\mathrm{pediatric}=\mathrm{false}}$
 \State Append $(\mathcal D=R,x=S\cup B,B,\text{opening})$ to $C$
\EndWhile
\State \Return $C$
\end{algorithmic}
\end{algorithm}

Neither outer loop has a global attempt limit. Neither
algorithm adds a background label-preservation check after acceptance.
The ePOCT+ path checks and DDXPlus schema merge should therefore be
read as the implemented consistency rules, not a proof of clinical
validity or unique identifiability of every target set.

%% file: appendix/patient_generation/elicitation_prompts.tex
\subsection{Association and Complaint-Phrase Elicitation}
\label{app:comor_elicitation}

The stored co-occurrence ratios and opening-phrase dictionary were
prepared by the resource-construction script \texttt{build\_joint.py}.
The templates below document this offline preparation. Its default
model is \texttt{gpt-5.6}; coefficient elicitation uses two rollouts and
batches of 20 pairs by default. Positive numerical responses are
aggregated in log space, with invalid responses omitted and the number
of valid responses recorded. The phrase dictionary is elicited in
batches of 50 evidence questions. The current generator reads the
saved tables; it does not repeat these calls for each patient.

\lstdefinestyle{comoroffline}{
 basicstyle=\footnotesize\ttfamily,
 columns=fullflexible,breaklines=true,breakatwhitespace=true,
 keepspaces=true,showstringspaces=false,
 frame=single,framerule=0.3pt,framesep=4pt,
 aboveskip=6pt,belowskip=8pt,
 literate={—}{{---}}1
}

\par\smallskip\noindent
\begin{minipage}{\linewidth}
\paragraph{Association multiplier.}
The first paragraph is the system message and the remainder is the
user template. The numbered pair list expands to the current batch.
\begin{lstlisting}[style=comoroffline]
You are a careful clinical epidemiologist with deep knowledge of disease comorbidity and shared etiology. Output ONLY the requested JSON — no prose.

For each numbered pair of conditions, estimate the CO-OCCURRENCE FACTOR: how many times more (or less) likely the two occur TOGETHER in the same patient vs if they were statistically independent.
  1.0 = independent; >1 comorbid (e.g. 20 HIV & tuberculosis, 6 COPD exacerbation & pulmonary embolism); <1 rarely together.
Judge ONLY association strength from real comorbidity/shared etiology; it is INDEPENDENT of how common each disease is.

0. "<disease A>"  +  "<disease B>"

Return ONLY a JSON object mapping each pair number to its factor, e.g. {"0":1.0,"1":8.5}. Include every number 0..<last index>.
\end{lstlisting}
\end{minipage}\par

\par\smallskip\noindent
\begin{minipage}{\linewidth}
\paragraph{Chief-complaint phrase dictionary.}
The first line is the system message and the remainder is the user
template. This operation converts questionnaire text into a stored
opening phrase; merged clinical findings retain their source questions
and answers.
\begin{lstlisting}[style=comoroffline]
Clinical terminology normalizer. Output ONLY JSON.

Convert each clinical questionnaire question into a SHORT chief-complaint phrase: a 1-4 word lowercase patient-symptom noun phrase, no question mark. Examples: 'Do you have a cough?' -> cough ; 'Do you feel your heart is beating fast?' -> palpitations ; 'Is your skin much paler than usual?' -> pallor.
Return ONLY a JSON object mapping each evidence code to its phrase.

<evidence code>	<question text>
\end{lstlisting}
\end{minipage}\par

%% file: appendix/comorbidity.tex
\section{Comorbidity versus Multimorbidity}
\label{app:comorbidity}
\label{app:comorbidity:terms}

\emph{Comorbidity} and \emph{multimorbidity} describe coexisting health
conditions from different perspectives. Comorbidity takes a specified
\emph{index disease} as its reference and considers additional conditions in
relation to it. Multimorbidity describes the presence of two or more conditions
without assigning one this privileged role \citep{valderas2009defining}.
For a patient with diabetes and asthma, a study focused on diabetes would
treat asthma as a comorbidity, whereas considering both conditions without
an index disease would describe multimorbidity. The same patient can thus be
described from either perspective.

Neither term specifies the relationship between the conditions. Coexisting
diseases may be causally related, share risk factors, or occur together by
chance; comorbidity does not require a causal link, and multimorbidity does not
imply statistical independence \citep{valderas2009defining}. Similarly, symptom
overlap is a separate property: sharing a finding does not turn multimorbidity
into comorbidity, and having distinct findings does not establish independence.
These properties should be described directly as disease association or
finding overlap, rather than used to define the two terms.

In our benchmark, the target $\mathcal D$ is an unordered set of concurrent
diagnoses, and the task is to recover all its members. Cases with
$k=|\mathcal D|\geq 2$ therefore concern multiple conditions without a
designated index disease. Because the sources include acute conditions, we
use \emph{concurrent diagnoses} to describe this task without restricting it to
chronic disease. The \emph{comorbidity sampler} models associations between
diagnoses during case construction; its sequential ordering does not make the
first sampled diagnosis clinically primary. The presenting complaint likewise
opens the consultation without giving one diagnosis priority in the target
set or scoring.

%% file: appendix/experiments/experiments.tex
\clearpage
\section{Experimental Protocols, Prompts, and Analysis}
\label{app:experiments}

\input{appendix/experiments/comor_setup.tex}

\input{appendix/experiments/comor_protocol.tex}

%% file: appendix/experiments/comor_setup.tex
\subsection{Data, Models, and Decoding}
\label{app:comor_setup}
\label{app:models}

The main evaluation takes all 200 generated records in each cardinality
stratum, and the oracle-count and single-diagnosis controls take the first
50, keeping case identities fixed for comparisons. The intended
doctor roster uses the following provider identifiers:
\begin{quote}\small\ttfamily
openai/gpt-5.6-sol\\
google/gemini-3.8-flash\\
qwen/qwen3.8-flash\\
google/gemma-4-26b-a4b-it\\
z-ai/glm-5.3-flash\\
deepseek/deepseek-v4.1-flash
\end{quote}
\paragraph{Model selection.}
\label{app:model_selection}
GPT-5.6 and Gemini-3.8-Flash represent frontier closed models. On the Arena
text leaderboard (open-model ranking, 13 September 2026;
\url{https://arena.ai}), GLM and DeepSeek are among the strongest open models,
and the Gemma-4 and Qwen3.8 families are among the strongest below 100B
parameters, so the roster spans the best open models and strong smaller ones.

We run these models with reasoning off or at minimal effort, since extra
test-time reasoning does not prevent the loss of accuracy in multi-turn
conversations. Reasoning models deteriorate as much as non-reasoning ones, and
their longer responses introduce more unwarranted
assumptions~\citep{laban2025lost}. Running without reasoning also keeps the
six-model evaluation affordable, since reasoning multiplies the generated tokens.

\paragraph{Reasoning and decoding.}
Reasoning is switched off where the provider allows it and otherwise set to
the lowest available effort (Table~\ref{tab:reasoning}). Every completion is
limited to 16,000 output tokens and uses the provider's default temperature.
The patient always runs with reasoning disabled. Evaluation logs retain model
identifiers, task arguments and generation events. The main scorer matches
candidate names and does not use an LLM judge.

\begin{table}[h]
\centering
\small
\caption{Reasoning settings of the doctor models in the main evaluation.}
\label{tab:reasoning}
\begin{tabular}{lcc}
\toprule
Model & DDXPlus & ePOCT+ \\
\midrule
GPT-5.6             & none             & none \\
Gemini-3.8-Flash    & minimal          & minimal \\
Qwen3.8-Flash       & disabled         & disabled \\
Gemma-4             & disabled         & disabled \\
GLM-5.3-Flash       & minimal   & minimal \\
DeepSeek-V4.1-Flash & disabled         & disabled \\
\bottomrule
\end{tabular}

\end{table}

\paragraph{Patient selection.}
\label{app:patient_screen}
The patient model is chosen before any doctor is evaluated. Each candidate
plays the patient for ten records stratified over $k=1$ to $4$ and answers six
probes per record: three about findings present in the record, which should be
answered \emph{yes}, and three about findings absent by construction, which
should be answered \emph{no}. The score is the mean of the two accuracies; we
also count answers that name the patient's own diagnosis without support in
the record. The rule selects the cheapest candidate within $0.05$ of the best
score. DeepSeek-V4.1-Flash is the cheapest candidate in that band
(Table~\ref{tab:patient_screen}). It is also one of the doctor models, and in interaction it is in the lower half of exact-set recovery on both sources (Table~\ref{tab:main}), so sharing the patient gives it no visible advantage. The
screen checks factual answering on a stratified sample rather than every
consultation; because the same patient serves every doctor, any residual
error is shared across models.

\begin{table}[h]
\centering
\small
\caption{Patient-selection screen: accuracy on present and absent findings,
their mean, and cost per probe. No candidate named an unsupported diagnosis.}
\label{tab:patient_screen}
\begin{tabular}{lcccc}
\toprule
Candidate & Present & Absent & Score & USD / probe \\
\midrule
Gemini-3.8-Flash              & 1.00 & 0.93 & 0.97 & $1.4\times10^{-3}$ \\
Qwen3.8-Flash                 & 1.00 & 0.87 & 0.93 & $2.2\times10^{-4}$ \\
\textbf{DeepSeek-V4.1-Flash}  & 1.00 & 0.83 & 0.92 & $1.5\times10^{-4}$ \\
Qwen3.8-27B                   & 1.00 & 0.83 & 0.92 & $1.6\times10^{-3}$ \\
Ling-3.0-Flash                & 1.00 & 0.73 & 0.87 & $1.8\times10^{-5}$ \\
GPT-5.6-Luna                  & 0.97 & 0.77 & 0.87 & $1.3\times10^{-4}$ \\
Qwen3.7-Flash                 & 1.00 & 0.67 & 0.83 & $1.4\times10^{-4}$ \\
DeepSeek-V4-Flash             & 0.97 & 0.67 & 0.82 & $2.6\times10^{-5}$ \\
\bottomrule
\end{tabular}
\end{table}

%% file: appendix/experiments/comor_protocol.tex
\subsection{Consultation Protocol and Task Mapping}
\label{app:comor_protocol}

The static prompt adapts MediQ's non-interactive setting
\citep{li2024mediq}; the free-text doctor and patient prompts adapt
AgentClinic \citep{schmidgall2024agentclinic}. Both request a diagnosis
set from the supplied vocabulary. The interactive doctor is explicitly
told that several diseases can coexist and to continue investigating
after identifying one likely condition. The oracle-count prompt replaces
this instruction with the true cardinality and an exact-count requirement;
the single-diagnosis control requests the single most likely diagnosis.
All use the same final-answer marker; set predictions are separated by
semicolons. Appendix~\ref{app:sim:prompts} reproduces the implemented
templates and Appendix~\ref{app:sim:deviations} identifies the adaptations.

\paragraph{Stopping and intermediate outputs.}
For trajectory analyses, we record the natural stopping time $\tau$
and final prediction, remove the commitment from the dialogue, and
instruct the doctor to continue toward the cap. Refusals are logged.
A separate call to the same model at each observed prefix elicits a
supported diagnosis set $\widetilde{\mathcal D}_t$; these probe outputs
are never fed back. We distinguish them from the committed prediction
$\widehat{\mathcal D}$ when comparing recovery before and after
continuation (Appendix~\ref{app:comor_analysis}).

\begin{table}[htbp]
\centering\small
\caption{Task entry points in \texttt{experiments/tasks.py}. All diagnosis
outputs use the shared candidate-name scorer.}
\label{tab:comor_tasks}
\begin{tabular}{@{}p{0.26\linewidth}p{0.30\linewidth}p{0.38\linewidth}@{}}
\toprule
Task & Doctor prompt & Evidence supplied \\
\midrule
\texttt{full} & \texttt{DX\_SYSTEM}, \texttt{dx\_user} & \raggedright Opening and all retained findings. \tabularnewline
\texttt{tree} & Same as \texttt{full} & \raggedright Opening and answers along the stored tree path. \tabularnewline
\texttt{replay} & Same as \texttt{full} & \raggedright Opening and recorded question--answer pairs. \tabularnewline
\texttt{interactive} & \texttt{free\_doctor\_sys} & \raggedright Dialogue with \texttt{patient\_sys}. \tabularnewline
\texttt{interactive\_single} & \texttt{single\_doctor\_sys} & \raggedright Same patient; single-diagnosis instruction. \tabularnewline
\texttt{interactive\_count} & \texttt{count\_doctor\_sys} & \raggedright Same patient; true count disclosed. \tabularnewline
\texttt{interactive\_menu} & \texttt{DOCTOR\_SYS} & \raggedright Deterministic answers to evidence-ID queries. \tabularnewline
\bottomrule
\end{tabular}
\end{table}

Population tasks take the number of cases per cardinality as an argument
(\texttt{per\_k}); the main evaluation uses 200 and the controls use 50. The single-diagnosis
task defaults to $k=1$ with probes disabled; comparisons at other cardinalities
must explicitly override this filter. Other consultation tasks default to
one probe per prefix, with the run argument controlling whether probes are
collected. Replay omits forced turns by default, deduplicates identical
question--answer strings, and skips source samples without recorded turns.
The same doctor model must be selected for the paired replay run.

At each free-text turn, the patient model receives its fixed record
and the current doctor questionp. It is instructed to answer in dialogue using one to three
sentences and not name diseases. The doctor retains the accumulating
conversation. The question budget counts doctor messages, which may
contain more than one interrogative sentence. In the menu condition,
one evidence-ID query is one question, and the deterministic answer
function returns a negative answer when that evidence is absent.

At the first parsed commitment, the system records the natural set
and its question count $\tau$. Trajectory runs remove that commitment
from the doctor's context and request further questions. Three
consecutive refusals or invalid actions terminate the loop. A consultation
with no commitment receives an explicit final-answer request. Natural
predictions remain the scored outputs regardless of subsequent continuation.
Algorithm~\ref{alg:comor_consultation} specifies this loop, including
termination before the nominal question cap.

Independent probes receive the opening and the question--answer
prefix and request currently supported diagnoses, permitting an empty
set. The usual trajectory configuration has one probe per observed
prefix, including the opening-only prefix; controls can disable probes.
No probe is fed back. These are outputs under a separate readout
prompt, not observations of a latent clinical posterior.

Demographic tree splits use opening information, and repeated evidence
IDs count as one question. Natural-stage comparisons exclude forced
turns; plotted trajectories require the reported support threshold.
The reference-tree configuration and information-gain estimator are
specified in Appendix~\ref{app:comor_eig}.

\input{appendix/experiments/consultation_algorithm.tex}
\input{appendix/experiments/analysis_definitions.tex}

\input{appendix/experiments/evidence_matched.tex}
\input{appendix/experiments/info_gain.tex}
\input{appendix/experiments/prompts.tex}

%% file: appendix/experiments/consultation_algorithm.tex
\begin{algorithm}[H]
\caption{Consultation, natural commitment, and independent probes}
\label{alg:comor_consultation}
\begin{algorithmic}[1]
\Require Opening $o$, retained record $S$, doctor, answer channel, cap $T=20$, probe count $m$
\State Initialize doctor context with its system prompt and $o$
\State $Q\gets[\,]$; $\mathrm{natural}\gets\bot$; $f\gets0$
\While{$|Q|<T$ and $f<3$}
 \State Update the free-text question counter, if applicable; generate a doctor reply
 \State Parse a nonempty diagnosis commitment, a question, or an invalid action
 \If{the reply is a question}
  \State Answer using the record and current question (LLM patient) or evidence ID (menu)
  \State Append the question--answer pair to $Q$, tagged forced iff $\mathrm{natural}\ne\bot$
  \State Add the answer to the doctor's context; $f\gets0$
 \ElsIf{the reply is a commitment and $\mathrm{natural}=\bot$}
  \State Store $\mathrm{natural}\gets(\tau=|Q|,\widehat{\mathcal D}=\text{parsed labels})$
  \State Remove this assistant reply; append the channel's continuation instruction
 \Else
  \State Log a refusal or invalid action; $f\gets f+1$
  \State Remove this assistant reply; append the channel's correction instruction
 \EndIf
\EndWhile
\If{$\mathrm{natural}=\bot$}
 \State Request a final answer, parse it, and store $(\tau=|Q|,\widehat{\mathcal D})$
 \State Mark the final answer as forced; flag a parsing failure if applicable
\EndIf
\For{$t=0,\ldots,|Q|$}
 \For{$j=1,\ldots,m$}
  \State Call the same doctor model afresh with the probe prompt, $o$, and $Q_{1:t}$
  \State Store $\widetilde{\mathcal D}_{t,j}$ without updating the consultation
 \EndFor
\EndFor
\State \Return $\widehat{\mathcal D}$ for scoring, $\tau$, $Q$, probes, and protocol events
\end{algorithmic}
\end{algorithm}

\paragraph{Parsing and termination details.}
The voluntary-commit branch requires a nonempty parsed label list.
Consequently, \texttt{Diagnosis Ready: none} is a valid empty probe
response but does not trigger voluntary stopping: free-text parsing
treats it as dialogue, while menu parsing treats it as invalid.
The fallback final-answer request is issued both at the cap and after
three consecutive invalid actions without a commitment. The log field
\texttt{budget\_exhausted} marks this fallback in both cases; the turn
count and events distinguish them. Static and menu fallback outputs
also accept a bare JSON list; free-text fallback outputs require the
diagnosis marker. Unparseable final outputs are scored as empty sets
and recorded as parsing failures.

%% file: appendix/experiments/analysis_definitions.tex
\subsection{Scoring and Diagnostic Trajectories}
\label{app:comor_analysis}

\paragraph{Set scores.}
Let $u=|\widehat{\mathcal D}\cap\mathcal D|$ and
$\widehat k=|\widehat{\mathcal D}|$. Besides the Jaccard and exact-set
metrics in Equation~\ref{eq:set_metrics}, we compute
\begin{equation}
 P=\frac{u}{\max(1,\widehat k)},\qquad
 R=\frac{u}{k},\qquad F_1=\frac{2u}{\widehat k+k}.
\end{equation}
All are computed per case before averaging. The four outcome categories
are exhaustive and disjoint: exact if $\widehat{\mathcal D}=\mathcal D$;
partially correct if $\emptyset\ne\widehat{\mathcal D}\subsetneq\mathcal D$;
partially incorrect if $u>0$ and $\widehat k>u$; and wrong if $u=0$.
Candidate names are matched case-insensitively and duplicates removed;
unmatched strings remain predicted labels and count as false positives.
Jaccard is computed from these stored predictions during analysis.

\paragraph{Recovery and retention.}
For the one-probe-per-prefix setting, $\widetilde{\mathcal D}_t$ denotes
the probe output after $t$ question--answer pairs. Let $H$ be the number
of observed pairs, including continuation, and let $\tau\le H$ be the
recorded stopping point. The true labels appearing in any probe through
the stopping point are
\begin{equation}
 E_\tau=\mathcal D\cap\bigcup_{t=0}^{\tau}\widetilde{\mathcal D}_t.
 \label{eq:comor_ever}
\end{equation}
Probe recall at $t$ is $|\mathcal D\cap\widetilde{\mathcal D}_t|/k$;
ever-recovered recall is $|E_\tau|/k$. We distinguish omissions at the
stopping-point probe from omissions in the committed prediction:
\begin{equation}
 L_{\mathrm{probe}}=\frac{|E_\tau\setminus\widetilde{\mathcal D}_\tau|}{k},
 \qquad
 L_{\mathrm{commit}}=\frac{|E_\tau\setminus\widehat{\mathcal D}|}{k}.
 \label{eq:comor_omission}
\end{equation}
The conditional omission fraction divides the first numerator by
$|E_\tau|$ and is undefined when $E_\tau=\emptyset$. The first comparison
uses the same readout prompt at both times; the second also changes the
readout prompt. Continuation analyses use the observed $H$, which can
be less than 20. Taking a union over several stochastic readouts gives
more opportunities to name a label; these measures describe output
retention rather than a model's unobserved internal state.

\subsection{Question Selection and Empirical Information Gain}
\label{app:comor_eig}

For a menu prefix $(q_1,a_1),\ldots,(q_t,a_t)$, let $I_t$ contain the
training-population records whose deterministic answers match every
observed answer. The opening does not enter this filter. For any
nonempty record set $I$, define
\begin{equation}
 p_d(I)=\frac{1}{|I|}\sum_{i\in I}\mathbf1\{d\in\mathcal D_i\},
 \qquad H(I)=\sum_{d\in\mathcal L}h_2(p_d(I)),
\end{equation}
where $h_2(p)=-p\log_2p-(1-p)\log_2(1-p)$ with $0\log_2 0=0$.
This is a sum of marginal binary entropies. With
$I_{t,q,a}=\{i\in I_t:\operatorname{answer}(i,q)=a\}$, the empirical
gain is
\begin{equation}
 G_t(q)=H(I_t)-\sum_{a:|I_{t,q,a}|>0}
 \frac{|I_{t,q,a}|}{|I_t|}H(I_{t,q,a}),\qquad
 G_t^*=\max_{q\notin\{q_1,\ldots,q_t\}}G_t(q).
 \label{eq:comor_eig}
\end{equation}
The selected question is compared with $G_t^*$ at the same prefix.
Prefixes with $|I_t|<30$ have unavailable gains, not zero gains.
Respiratory summaries exclude forced turns and compare mean selected
gain with mean best-available gain over supported prefixes; their ratio
is a ratio of means. The reference tree uses multilabel entropy, a
maximum depth of 14, and a minimum leaf size of 20 on 8,000 generated
training cases. Its direct leaf prediction and its question path are
distinct from the \texttt{tree} task's subsequent LLM readout.

\subsection{Realized Finding Overlap and Presenting-Diagnosis Recovery}
\label{app:comor_overlap}

\paragraph{Finding overlap.}
For a DDXPlus case, let $F_d$ contain the variable--answer pairs from
the selected source record for diagnosis $d$ that survive in the merged
signal, excluding antecedents. Define contributor sets and distinct
clinical variables as
\begin{equation}
 A(v,a)=\{d:(v,a)\in F_d\},\qquad
 V=\{v:\exists a,\ A(v,a)\ne\emptyset\}.
\end{equation}
The realized overlap is
\begin{equation}
 O=\frac{|\{v\in V:\exists a,\ |A(v,a)|\ge2\}|}{|V|}.
 \label{eq:comor_overlap}
\end{equation}
A variable counts once even when it has several answers; sharing
requires the same retained answer in at least two components. Cases
without clinical variables have undefined overlap. Components are
recovered by replaying generation with the recorded seed and checking
the regenerated cases against the stored population. Associations with
recovery or omission use Spearman correlation separately within each
$k$. Generic disease-profile membership does not replace this
case-specific attribution.

\paragraph{Presenting-diagnosis advantage.}
For a case with $k>1$, let $d_c\in\mathcal D$ be the component selected
to supply the opening complaint. For information condition $m$, define
\begin{equation}
 g_m=\mathbf1\{d_c\in\widehat{\mathcal D}_m\}
 -\frac{|(\mathcal D\setminus\{d_c\})\cap\widehat{\mathcal D}_m|}{k-1},
 \qquad \Delta g=g_{\mathrm{interactive}}-g_{\mathrm{full}}.
 \label{eq:comor_presenting_gap}
\end{equation}
These case-level contrasts give each case equal weight. The paired
contrast keeps the same case in both conditions. Since the generator
selects the presenting diagnosis using recorded severity, this is an
observational comparison; it does not identify a causal effect of the
opening. Integration analyses use 4,000 case-bootstrap draws with
seed 2026, preserving condition pairs and stratifying pooled contrasts
by $k$. Respiratory analyses use 5,000 draws with seed 0. Reported
95\% intervals are pointwise. We report intervals rather than single scores,
since a single run can misrank LLM systems~\citep{potamitis2025reasonbench}.

%% file: appendix/experiments/evidence_matched.tex
\subsection{Matched Respiratory Cohort and the Evidence Cross-Over}
\label{app:evidence_matched}

This appendix documents the experiment behind Section~\ref{sec:tree}
and Figure~\ref{fig:tree}. All numbers come from one run of GPT-5.6 on
DDXPlus with a budget of 20 questions.

\paragraph{Why this design.}
In the main consultations, a missed second diagnosis can come from the
questions, from the reasoning, or from the case itself, because cases with
more diagnoses also differ in other ways. The design removes these
confounders one at a time. Each two-diagnosis patient is paired with a
single-diagnosis patient built from the same source component and given the
same opening, so the only change from $k=1$ to $k=2$ is the added disease.
Which component opens is balanced within each pair, so the effect is not
tied to one disease always being the second. Questions come from a fixed
menu and are answered deterministically from the record, so each question
maps to a known finding and no patient simulator adds noise. The tree,
fitted on separate patients, shows what the same menu and budget can
achieve. The fixed reader scores any set of answers without the model's
reasoning, which isolates the value of the questions. Letting the model
read its own answers, the tree's, or the full record holds the evidence
fixed and isolates its reasoning.

\paragraph{Cohort.}
The cohort uses seven DDXPlus respiratory diagnoses (acute laryngitis,
acute otitis media, acute rhinosinusitis, bronchitis, chronic
rhinosinusitis, influenza, and viral pharyngitis). URTI is excluded as an
umbrella label; allergic sinusitis and bronchospasm are excluded for
having too few source profiles. All 21 pairs are feasible under the
generator. For each pair we build eight \emph{families} from the same
source components. A family holds the combined patient ($k=2$), opened
with the chief complaint of one component, and that component alone
($k=1$), with the same opening. Four families of each pair open on each
component, so every diagnosis opens 24 families. This gives
$21\times8\times2=336$ consultations, 168 at each $k$. Combining the
components can change some shared answers, so the paired patient is
matched on source and opening, not on every finding. The menu, the tree
and the reader use the DDXPlus evidence vocabulary only, so these cases
carry no background complaints. Source clinical
profiles are split into training, validation, and test partitions; test
profiles do not appear in the training data of the tree or the
classifier.

\paragraph{Consultation protocol.}
The doctor prompt is the AgentClinic-derived menu prompt of
Appendix~\ref{app:sim:prompts} with four adaptations. It states that more
than one diagnosis may be present, gives the closed candidate list,
selects questions by \texttt{ASK <id>} from a 49-item menu, and states
the budget of 20 questions. The full menu is listed in the prompt, so the
model and the tree choose from the same questions under the same budget.
There is no patient simulator, and answers are
deterministic functions of the record, with absent evidence answered
negatively. A commitment before the budget ends the consultation; at the
budget a final set is requested. GPT-5.6 commits on its own in 2 of 336
consultations, both after 19 questions, and asks all 20 questions in the
other 334. It returned no invalid action and no unparseable set.

\paragraph{Reference tree.}
The tree is a multilabel decision tree fitted on 10,752 training
records (entropy criterion, at least two records per leaf). It was
selected on the validation partition among eleven fits, and the selected
tree has maximum depth 23. It runs under the same 20-question budget and
stops at a leaf, which it reaches before the budget on every test path.
It asks 10.0 questions at $k=1$ and 7.5 at $k=2$ on average. One tree
serves both $k$; the true count is never supplied. It is a learned
reference, not an optimal policy. At each node it asks the menu question
that most reduces the uncertainty about the diagnosis set among the training
patients. The candidate sets are fixed and test patients come from the same
generator, so it is a reasonable reference for what the menu can reveal
within the budget. On validation and test patients the tree comes within 4
to 7 points of a random forest that reads the full record.

\paragraph{Tree depth.}
A deep tree could in principle ask for the whole record. We refit the
same recipe with a maximum depth between 3 and 18
(Table~\ref{tab:tree_depth}). Validation accuracy rises until depth 8 to 10
and then stays flat. The unrestricted tree asks about ten questions and is
no more accurate than a depth-10 tree. The reference is therefore not
strong because it reads most of the record, and the comparison in
Section~\ref{sec:tree} does not depend on the depth chosen.

\begin{table}[H]
\centering
\small
\begin{tabular}{rrrrrrrr}
\toprule
& & \multicolumn{2}{c}{Tree leaf} & \multicolumn{2}{c}{Fixed reader} & \multicolumn{2}{c}{Questions} \\
\cmidrule(lr){3-4}\cmidrule(lr){5-6}\cmidrule(lr){7-8}
Max depth & Leaves & $k=1$ & $k=2$ & $k=1$ & $k=2$ & $k=1$ & $k=2$ \\
\midrule
3 & 8 & 58.0 & 11.3 & 49.7 & 11.6 & 3.0 & 3.0 \\
5 & 28 & 83.3 & 40.5 & 87.5 & 45.5 & 4.9 & 4.9 \\
6 & 47 & 87.8 & 54.8 & 88.4 & 55.4 & 5.4 & 5.5 \\
8 & 109 & 92.9 & 68.5 & 90.2 & 64.9 & 6.6 & 6.3 \\
10 & 203 & 95.5 & 77.4 & 90.5 & 66.7 & 8.0 & 7.0 \\
12 & 305 & 94.6 & 74.4 & 89.6 & 64.6 & 8.7 & 7.5 \\
15 & 436 & 93.8 & 77.4 & 90.2 & 64.9 & 9.8 & 7.7 \\
18 & 542 & 93.5 & 78.0 & 90.5 & 64.6 & 10.0 & 7.8 \\
23 (reference) & 593 & 92.3 & 76.8 & 90.8 & 66.1 & 10.0 & 7.8 \\
\bottomrule
\end{tabular}
\caption{Tree depth on the validation partition with a 20-question budget.
Exact-set recovery (\%) of the tree's leaf and of the fixed reader on the
tree's answers, and mean questions asked. The reference tree is the
unrestricted fit.}
\label{tab:tree_depth}
\end{table}

\paragraph{Fixed reader.}
The fixed reader is a set of seven one-versus-rest logistic regressions
fitted on the same training records. Each training record is presented
under three masks, the opening alone, the full record, and the opening
plus one to nine random menu answers, so that an unasked question and a
negative answer are distinct inputs. Given a set of collected answers,
the reader returns every diagnosis with probability of at least one half.
On validation records with random answers its accuracy keeps rising beyond
nine answers (at $k=2$, 27\% with nine, 50\% with twenty, and 71\% with the
full record), so we use it unchanged for 20-question evidence. It is trained on
the source population and its absolute accuracy is a source-relative
ceiling, not a clinical one; the analysis uses only the contrast between
two evidence sets read by the same reader.

\paragraph{Readouts.}
The self, tree, and full readouts are separate GPT-5.6 calls with the
MediQ-derived static prompt of Appendix~\ref{app:sim:prompts}. Each
receives the opening and, as question--answer pairs in one format, the
answers GPT-5.6 collected in its own consultation, the answers along the
tree's path for the same patient, or every menu answer. The self-readout
therefore tests whether converting the dialogue into a clean list changes
the prediction. It lowers exact-set recovery by three to four points
(Table~\ref{tab:evidence_arms}), which is why the GPT-5.6 bars in
Figure~\ref{fig:tree}a and the first bars in Figure~\ref{fig:tree}c differ
slightly.

\paragraph{Design.}
Table~\ref{tab:evidence_design} lists the cells of the cross-over.
Feeding GPT-5.6's evidence to the tree itself is not possible, because the
tree requires the answers along its own path and GPT-5.6 asks only
43\% ($k=1$) and 38\% ($k=2$) of the patient's path questions; the fixed
reader stands in for the tree in that cell.

\begin{table}[H]
\centering
\small
\begin{tabular}{lll}
\toprule
Evidence from & Diagnosed by & Reported as \\
\midrule
GPT-5.6's questions & GPT-5.6 in the consultation & Figure~\ref{fig:tree}a \\
Tree's questions & Tree leaf & Figure~\ref{fig:tree}a \\
GPT-5.6's questions & Fixed reader & Figure~\ref{fig:tree}b \\
Tree's questions & Fixed reader & Figure~\ref{fig:tree}b \\
GPT-5.6's questions, as a list & Fresh GPT-5.6 readout & Figure~\ref{fig:tree}c \\
Tree's questions, as a list & Fresh GPT-5.6 readout & Figure~\ref{fig:tree}c \\
Full record & Fresh GPT-5.6 readout & Figure~\ref{fig:tree}c \\
\bottomrule
\end{tabular}
\caption{Cells of the evidence cross-over. Rows one and two are the
direct comparison; rows three and four vary the evidence with the reader
fixed; rows five to seven vary the evidence with GPT-5.6 as the reader.}
\label{tab:evidence_design}
\end{table}

\paragraph{Per-label scores.}
Table~\ref{tab:evidence_arms} reports exact-set recovery, label recall,
and false positives per consultation for every cell. Two facts do not
show in the exact-set figure. First, GPT-5.6 over-calls. At $k=1$ it
returns two or more diagnoses in 42\% of consultations and 45\% of
full-record readouts. Second, at $k=2$ its consultations return one
diagnosis in 24\% of cases, two in 60\%, and three or more in 16\%; among
the two-diagnosis answers the pair is correct in 29 of 100. With the full
record the single-diagnosis share falls to 10\%, but 38\% of readouts then
return three or more diagnoses, and the pair is correct in 44 of 87
two-diagnosis readouts.

\begin{table}[H]
\centering
\small
\setlength{\tabcolsep}{4pt}
\begin{tabular}{llrrrrrr}
\toprule
& & \multicolumn{3}{c}{$k=1$} & \multicolumn{3}{c}{$k=2$} \\
\cmidrule(lr){3-5}\cmidrule(lr){6-8}
Evidence from & Reader & Exact & Recall & FP & Exact & Recall & FP \\
\midrule
GPT-5.6's questions & GPT-5.6, in consultation & 35.7 & 69.6 & 0.75 & 17.3 & 58.3 & 0.77 \\
Tree's questions & Tree leaf & 92.9 & -- & -- & 81.5 & -- & -- \\
GPT-5.6's questions & Fixed reader & 73.8 & 91.7 & 0.21 & 45.8 & 77.1 & 0.15 \\
Tree's questions & Fixed reader & 89.3 & 92.9 & 0.04 & 73.2 & 90.5 & 0.14 \\
Full record & Fixed reader & 85.7 & 96.4 & 0.11 & 78.6 & 94.0 & 0.11 \\
GPT-5.6's questions, as a list & GPT-5.6 readout & 32.7 & 70.2 & 0.71 & 13.7 & 61.0 & 0.79 \\
Tree's questions, as a list & GPT-5.6 readout & 47.0 & 81.0 & 0.54 & 23.2 & 65.2 & 0.60 \\
Full record & GPT-5.6 readout & 42.3 & 85.7 & 0.61 & 26.2 & 78.0 & 0.77 \\
\bottomrule
\end{tabular}
\caption{Exact-set recovery (\%), label recall (\%), and false positives
per consultation for every cell of the cross-over; 168 consultations per
$k$. The tree's leaf prediction is scored on exact-set recovery only.}
\label{tab:evidence_arms}
\end{table}

\paragraph{Question counts and coverage.}
GPT-5.6 asks two to three times as many questions as the tree, so the comparison in
Figure~\ref{fig:tree}b favours GPT-5.6 in the amount of evidence. Its
coverage of the tree's path for the same patient falls from 42.6\% at
$k=1$ to 38.3\% at $k=2$ (change $-4.3$ points, 95\% interval $-6.4$ to
$-2.1$). Figure~\ref{fig:evidence_budget} removes the difference in
length. The fixed reader reads GPT-5.6's first $n$ answers and the tree's
path capped at $n$ questions. The tree reaches its plateau after five
questions at $k=1$ and eight at $k=2$. GPT-5.6's evidence improves
slowly over the consultation and stays below the tree's from the third
question on, and the gap is larger at $k=2$. After eight questions, the
reader's exact-set recovery from GPT-5.6's evidence is 50\% ($k=1$) and
17\% ($k=2$), against 92\% and 72\% from the tree's. All intervals are 95\% bootstraps
over whole families within diagnosis pair, conditional on the 21 pairs
and the fitted reader.

\begin{figure}[H]
\centering
\includegraphics[width=\linewidth]{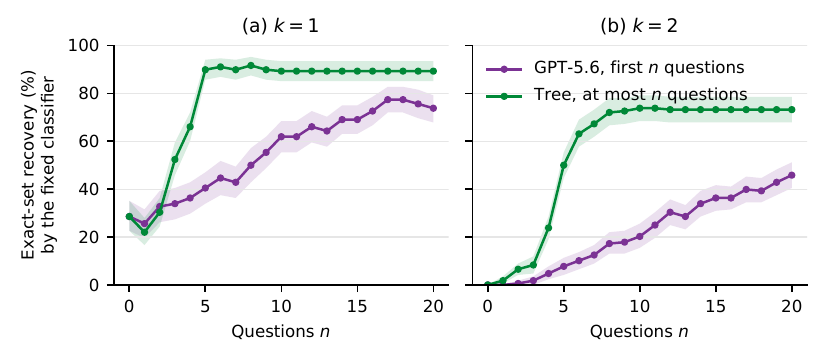}
\caption{Exact-set recovery by the fixed reader from GPT-5.6's first $n$ answers
and from the tree capped at $n$ questions. Bands are 95\% intervals.}
\label{fig:evidence_budget}
\end{figure}

%% file: appendix/experiments/info_gain.tex
\subsection{Information Gained over the Consultation}
\label{app:info_gain}

This analysis follows the consultations of Appendix~\ref{app:evidence_matched}
question by question and asks how much each system has learned about the
diagnosis set at every point. It uses the same 336 GPT-5.6 consultations
and the tree's paths on the same patients, and it makes no model calls.

\paragraph{Belief model.}
We need an observer that turns a history of answers into a distribution
over diagnosis sets. The hypothesis space $\mathcal S$ holds the 28
possible sets (7 single diagnoses and 21 pairs); the true $k$ is not
given. We fit a naive Bayes model on the training split (4,032 records
after removing the duplicate opening of each pair), with one class per set.
Its features are sex, age decade, and the 49 menu questions. Binary
questions are yes or no, categorical and ordinal questions take their value
or \emph{absent}, and each multi-select question becomes one yes or no
feature per value. Conditional tables use add-one smoothing, and the prior
gives $k=1$ and $k=2$ equal mass. After a history $h_t$ of $t$ answered
questions, the belief is $p(S\mid h_t)\propto p(S)\prod_{f\in h_t}p(a_f\mid S)$,
where $h_t$ always contains sex, age, and the opening complaint. From the
full record the observer recovers the exact set in 97.6\% of validation
patients at $k=1$ and 79.8\% at $k=2$. Its average belief on the true set
stays close to how often its most likely set is correct, for example 0.57
against 0.61 for GPT-5.6's evidence at $k=2$.

\paragraph{Measures.}
The information gained after $t$ questions is $H(S\mid h_0)-H(S\mid h_t)$ in
bits, where $h_0$ is the opening, and the belief on the true set is
$p(S^\ast\mid h_t)$. A repeated question adds nothing. After the tree
reaches a leaf it asks nothing more, so its last value is carried forward to
twenty questions. GPT-5.6 could end the consultation at any
turn but asked all twenty questions in 334 of 336 consultations
(Appendix~\ref{app:evidence_matched}), so its curves run to twenty. We
therefore compare the two systems both after twenty questions and over the
tree's own number of questions. Bands are 95\% bootstrap intervals over
families.

\paragraph{Results.}
Over the number of questions the tree asks, a tree question gains 0.42
($k=1$) and 0.53 ($k=2$) bits and a GPT-5.6 question 0.27 and 0.26
(Figure~\ref{fig:tree}c). After twenty questions GPT-5.6 puts 0.86 and 0.57
on the true set, against 0.97 and 0.76 for the tree and 0.97 and 0.81 for
the full record (Figure~\ref{fig:info_gain}b). A fixed classifier reading
the same evidence gives the same order (a). The observer and the tree are
fitted on the same training patients, but GPT-5.6's own reading of the
evidence (Figure~\ref{fig:tree}b) and the record-matching gain of
Appendix~\ref{app:comor_eig} agree with it.

\paragraph{Shared questions.}
Within the tree's own number of questions, GPT-5.6 asks 21\% of the tree's
questions at $k=1$ and 12\% at $k=2$, a difference of $-9$ $[-11, -7]$
points within family. The tree changes its path when the second disease
changes the answers, and GPT-5.6 mostly does not. Over the first five
questions, the questions asked for the two patients of a family overlap by
0.73 (Jaccard) for GPT-5.6 and 0.66 for the tree, against 0.06 by chance,
and the gap widens at seven questions.

\begin{figure}[h]
\centering
\includegraphics[width=\linewidth]{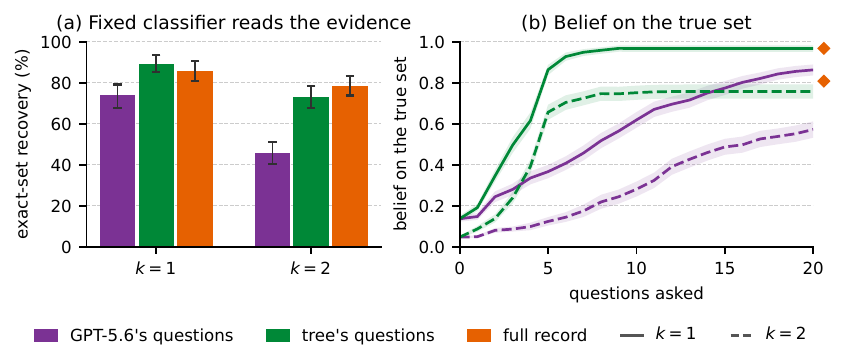}
\caption{GPT-5.6 and the tree on the matched respiratory cohort. (a)
Exact-set recovery when a fixed classifier reads each system's evidence or
the full record. (b) The observer's belief on the true set after each
question. After the tree stops, its last value is carried forward, and the
diamonds show the full record.}
\label{fig:info_gain}
\end{figure}

\paragraph{Limitations.}
The values are properties of the observer, not of GPT-5.6's internal
belief. Naive Bayes treats the values of a multi-select question as
independent given the set, which overstates how much such an answer is
expected to reveal before it is given. We therefore report the information
actually gained after each answer rather than the expected gain of each
question.

%% file: appendix/experiments/prompts.tex
\subsection{Prompt Provenance and Adaptations}
\label{app:sim:deviations}

The templates below are extracted from our repository's
\texttt{experiments/lib/prompts.py}; their task wiring is in
\texttt{experiments/tasks.py}. The source comments identify the
free-text doctor and patient as adaptations of AgentClinic
\citep{schmidgall2024agentclinic}. They retain the doctor/patient roles,
short dialogue responses, a question budget, and the instruction not
to reveal the patient's disease name. Our doctor adds an explicit
comorbidity instruction, the candidate diagnosis list, and a
semicolon-separated set commitment. Our patient uses the standard
AgentClinic patient prompt, without the mention of physical exams, and
receives the rendered case findings as its information. Like the original, this wording does not specify a deterministic
policy for every absent or unknown finding; the menu channel supplies that
control separately.

The local adaptation omits the \texttt{Request Test:} route and the
research, chain-of-thought, and target-language slots. The menu action
format, forced continuation, oracle-count instruction, and independent
probes are local protocol components. The single-diagnosis control
removes the comorbidity paragraph but retains our candidate list and
response format. AgentClinic defines only dialogue prompts, so the static
template used for full information adapts the non-interactive setting of
MediQ \citep{li2024mediq} to a diagnosis set and the shared commitment
marker. The static and dialogue prompts share the candidate list, the
statement that the patient may have several diseases, the answer format
and the scorer. The dialogue prompt adds the question budget and an
instruction to keep questioning after a first likely diagnosis. Reading
a consultation's own findings with the static prompt changes exact
recovery by at most four points (Table~\ref{tab:evidence_arms}). A moderator/judge template is present in the
source file but is not called by the main evaluation tasks.

\subsection{Implemented Prompt Templates}
\label{app:sim:prompts}

Angle-bracketed fields below stand for runtime values. Candidate lists,
finding lists, and menu entries expand to their complete case- or
source-specific contents. All other wording is taken from the local
prompt definitions. The single-diagnosis and oracle-count controls are
shown as exact substitutions to avoid repeating their shared text.

\lstdefinestyle{comorprompt}{
 basicstyle=\footnotesize\ttfamily,
 columns=fullflexible,breaklines=true,breakatwhitespace=true,
 keepspaces=true,showstringspaces=false,
 frame=single,framerule=0.3pt,framesep=4pt,
 aboveskip=6pt,belowskip=8pt,
 literate={—}{{---}}1
}

\par\smallskip\noindent
\begin{minipage}{\linewidth}
\paragraph{Static diagnosis.}
The first line is the system message; the remainder is the user message. The full-information, tree-guided, and replay tasks share this template, with different finding lists.
\begin{lstlisting}[style=comorprompt]
You are an experienced doctor trying to make a medical decision about a patient.

A patient comes into the clinic presenting with a symptom as described in the statements below:
Patient: <opening presentation>
- <retained finding>
- <additional retained findings>
Given the information from above, your task is to identify every diagnosis the patient has from the candidate list below. The patient may have more than one disease at the same time.
Candidate diagnoses:
<candidate diagnosis list>
To the best of your ability, answer with ONLY "Diagnosis Ready: [diagnosis 1]; [diagnosis 2]; ..." and nothing else.
\end{lstlisting}
\end{minipage}\par

\par\smallskip\noindent
\begin{minipage}{\linewidth}
\paragraph{Free-text doctor.}
System message. In the main evaluation and in Section~\ref{sec:failure}, the
question counter follows each patient answer as a new line,
``(You have asked $n$ of 20 questions so far.)'', so the system message stays
fixed. The earlier runs behind the oracle-count and single-diagnosis controls
(Appendices~\ref{app:oracle} and~\ref{app:single}) and the open-model probe
analyses (Appendices~\ref{app:probe_anchoring} and~\ref{app:compliance}) placed
it in the system message instead, as ``You have asked $n$ questions so far.''
after the first sentence below. Each such comparison uses runs of one version.
\begin{lstlisting}[style=comorprompt]
You are a doctor named Dr. Agent who only responds in the form of dialogue. You are inspecting a patient who you will ask questions in order to understand their disease. You are only allowed to ask 20 questions total before you must make a decision. Your questions must be 1-3 sentences in length.
The patient may have more than one disease at the same time. Do not stop investigating after finding one likely diagnosis; actively probe whether the findings are fully explained by a single condition or point to co-occurring conditions.
Once you have decided to make a diagnosis please say "Diagnosis Ready: [diagnosis 1]; [diagnosis 2]; ..." listing every disease the patient has, using ONLY names from this list:
<candidate diagnosis list>
Remember, you must discover their diseases by asking them questions.
\end{lstlisting}
\end{minipage}\par

\par\smallskip\noindent
\begin{minipage}{\linewidth}
\paragraph{Patient simulator.}
System message. The current doctor question is the user message. Each patient call receives these two messages.
\begin{lstlisting}[style=comorprompt]
You are a patient in a clinic who only responds in the form of dialogue. You are being inspected by a doctor who will ask you questions in order to understand your disease. Your answer will only be 1-3 sentences in length.
Below is all of your information.
Who you are: <opening presentation>
- <retained finding>
- <additional retained findings>
Remember, you must not reveal your disease explicitly but may only convey the symptoms you have in the form of dialogue if you are asked.
\end{lstlisting}
\end{minipage}\par

\par\smallskip\noindent
\begin{minipage}{\linewidth}
\paragraph{Menu doctor.}
System message. The question-list slot expands to all available evidence IDs and their questions.
\begin{lstlisting}[style=comorprompt]
You are a doctor named Dr. Agent. You are inspecting a patient who you will ask questions in order to understand their disease. You are only allowed to ask 20 questions total before you must make a decision.
The patient may have more than one disease at the same time. Do not stop investigating after finding one likely diagnosis; actively probe whether the findings are fully explained by a single condition or point to co-occurring conditions.
Questions you may ask (id: question):
<evidence id>: <question text>
Each turn reply with EXACTLY ONE line: ASK <id>
Once you have decided to make a diagnosis please say "Diagnosis Ready: [diagnosis 1]; [diagnosis 2]; ..." listing every disease the patient has, using ONLY names from this list:
<candidate diagnosis list>
\end{lstlisting}
\end{minipage}\par

\par\smallskip\noindent
\begin{minipage}{\linewidth}
\paragraph{Single-diagnosis control: changed suffix.}
The opening paragraph is identical to the free-text doctor template. Delete its comorbidity paragraph and replace the remaining suffix with the following text. The closed candidate list and local response format are retained.
\begin{lstlisting}[style=comorprompt]
Once you have decided to make a diagnosis please say "Diagnosis Ready: [diagnosis here]" giving the single most likely diagnosis, using ONLY a name from this list:
<candidate diagnosis list>
Remember, you must discover their disease by asking them questions.
\end{lstlisting}
\end{minipage}\par

\par\smallskip\noindent
\begin{minipage}{\linewidth}
\paragraph{Oracle-count control: replacement paragraph.}
Replace only the free-text doctor's comorbidity paragraph with the following. For $k=1$, the code uses the singular ``condition''; the other wording is unchanged.
\begin{lstlisting}[style=comorprompt]
The patient has exactly <k> distinct conditions at the same time. You must identify all of them — no fewer and no more than <k>.
\end{lstlisting}
\end{minipage}\par

\par\smallskip\noindent
\begin{minipage}{\linewidth}
\paragraph{Independent diagnosis probe.}
The first paragraph and candidate list form the system message; the remainder is the user message. At $t=0$, the question--answer lines are replaced by \texttt{(no questions asked yet)}.
\begin{lstlisting}[style=comorprompt]
You are a physician reviewing a partial consultation. Several diagnoses may be present at once.

Candidate diagnoses:
<candidate diagnosis list>

Patient: <opening presentation>

Consultation so far:
Q: <question 1>
A: <answer 1>
Q: <remaining questions>
A: <remaining answers>

Based ONLY on the information above, reply with "Diagnosis Ready: [diagnosis 1]; [diagnosis 2]; ..." listing the diagnoses currently supported by the evidence, chosen from the candidate list. If none are supported yet, reply "Diagnosis Ready: none".
\end{lstlisting}
\end{minipage}\par

\paragraph{Opening, answers, and protocol messages.}
These are user messages added by the task. The labels above each block are explanatory and are not sent to the model.

\par\smallskip\noindent
\begin{minipage}{\linewidth}
\noindent\textit{Opening user message}
\begin{lstlisting}[style=comorprompt]
Patient: <opening presentation>
Begin the consultation.
\end{lstlisting}
\end{minipage}\par

\par\smallskip\noindent
\begin{minipage}{\linewidth}
\noindent\textit{Menu answer user message}
\begin{lstlisting}[style=comorprompt]
<question text>: <answer>
\end{lstlisting}
\end{minipage}\par

\par\smallskip\noindent
\begin{minipage}{\linewidth}
\noindent\textit{Free-text continuation}
\begin{lstlisting}[style=comorprompt]
Do NOT give a final diagnosis yet. Continue the consultation and keep asking the patient questions.
\end{lstlisting}
\end{minipage}\par

\par\smallskip\noindent
\begin{minipage}{\linewidth}
\noindent\textit{Menu continuation}
\begin{lstlisting}[style=comorprompt]
Do NOT give a final diagnosis. Continue the consultation: keep asking questions from the list, one per turn, as ASK <id>.
\end{lstlisting}
\end{minipage}\par

\par\smallskip\noindent
\begin{minipage}{\linewidth}
\noindent\textit{Fallback final-answer request}
\begin{lstlisting}[style=comorprompt]
Question budget exhausted. Reply now with "Diagnosis Ready: [diagnosis 1]; [diagnosis 2]; ...".
\end{lstlisting}
\end{minipage}\par

\par\smallskip\noindent
\begin{minipage}{\linewidth}
\noindent\textit{Invalid free-text reply, before commitment}
\begin{lstlisting}[style=comorprompt]
Invalid reply. Ask the patient one question in plain dialogue or commit with "Diagnosis Ready: [diagnosis 1]; ...".
\end{lstlisting}
\end{minipage}\par

\par\smallskip\noindent
\begin{minipage}{\linewidth}
\noindent\textit{Invalid free-text reply, after commitment}
\begin{lstlisting}[style=comorprompt]
Invalid reply. Ask the patient one question in plain dialogue.
\end{lstlisting}
\end{minipage}\par

\par\smallskip\noindent
\begin{minipage}{\linewidth}
\noindent\textit{Invalid menu action, before commitment}
\begin{lstlisting}[style=comorprompt]
Invalid action. Reply with exactly one line: ASK <id> or "Diagnosis Ready: [diagnosis 1]; ...".
\end{lstlisting}
\end{minipage}\par

\par\smallskip\noindent
\begin{minipage}{\linewidth}
\noindent\textit{Invalid menu action, after commitment}
\begin{lstlisting}[style=comorprompt]
Invalid action. Reply with exactly one line: ASK <id>.
\end{lstlisting}
\end{minipage}\par

%% file: appendix/additional_results.tex
\section{Additional Results}
\label{app:results}
\label{app:additional_results}

\input{appendix/additional_results/ci.tex}

\input{appendix/additional_results/perk.tex}

\input{appendix/additional_results/oracle.tex}
\input{appendix/additional_results/interaction.tex}

\input{appendix/additional_results/single.tex}

\input{appendix/additional_results/evidence_support.tex}

\input{appendix/additional_results/designed_k2.tex}

\input{appendix/additional_results/sharing.tex}

\input{appendix/additional_results/probe_anchoring.tex}

\input{appendix/additional_results/compliance.tex}

%% file: appendix/additional_results/ci.tex
\subsection{Confidence intervals for the main tables}
\label{app:ci}

Tables~\ref{tab:main_ci} and~\ref{tab:exact_by_k_ci} give 95\% intervals for
every value of Tables~\ref{tab:main} and~\ref{tab:exact_by_k}. They are case
bootstraps with 4,000 draws and seed 2026, stratified by $k$, and the same
resampled cases are used for every model and condition of a source. Pooled
exact-set recovery is resolved to within $\pm0.01$ to $\pm0.03$, so models
that are close in Table~\ref{tab:main} are not separated. In interaction at
$k\ge3$, every model stays at or below 7\%, with upper bounds of at most 10\%.

\begin{table}[h]
\centering
\footnotesize
\setlength{\tabcolsep}{3pt}
\caption{Table~\ref{tab:main} with 95\% bootstrap intervals, 800 cases per source.}
\label{tab:main_ci}
\begin{tabular}{llcccc}
\toprule
 & & \multicolumn{2}{c}{Exact $\uparrow$} & \multicolumn{2}{c}{Jaccard $\uparrow$} \\
\cmidrule(lr){3-4}\cmidrule(lr){5-6}
Source & Model & F & I & F & I \\
\midrule
DDXPlus  & GPT-5.6             & 0.11 [0.09, 0.13] & 0.08 [0.07, 0.10] & 0.46 [0.44, 0.48] & 0.29 [0.27, 0.31] \\
         & Gemini-3.8-Flash    & 0.26 [0.24, 0.28] & 0.08 [0.06, 0.10] & 0.47 [0.45, 0.50] & 0.32 [0.30, 0.34] \\
         & Qwen3.8-Flash       & 0.17 [0.15, 0.20] & 0.06 [0.05, 0.08] & 0.46 [0.44, 0.48] & 0.27 [0.25, 0.29] \\
         & GLM-5.3-Flash       & 0.11 [0.09, 0.13] & 0.02 [0.01, 0.03] & 0.46 [0.44, 0.48] & 0.24 [0.23, 0.26] \\
         & DeepSeek-V4.1-Flash & 0.04 [0.03, 0.05] & 0.02 [0.01, 0.03] & 0.37 [0.35, 0.38] & 0.25 [0.23, 0.26] \\
         & Gemma-4               & 0.10 [0.08, 0.12] & 0.01 [0.01, 0.02] & 0.43 [0.41, 0.45] & 0.24 [0.22, 0.25] \\
\midrule
ePOCT+   & GPT-5.6             & 0.21 [0.19, 0.24] & 0.08 [0.06, 0.10] & 0.57 [0.55, 0.59] & 0.30 [0.28, 0.32] \\
         & Gemini-3.8-Flash    & 0.21 [0.18, 0.24] & 0.10 [0.08, 0.12] & 0.51 [0.49, 0.53] & 0.39 [0.37, 0.41] \\
         & Qwen3.8-Flash       & 0.18 [0.16, 0.21] & 0.06 [0.05, 0.08] & 0.49 [0.47, 0.51] & 0.24 [0.22, 0.26] \\
         & GLM-5.3-Flash       & 0.15 [0.13, 0.18] & 0.05 [0.04, 0.07] & 0.52 [0.50, 0.54] & 0.26 [0.25, 0.28] \\
         & DeepSeek-V4.1-Flash & 0.15 [0.13, 0.18] & 0.06 [0.05, 0.08] & 0.53 [0.52, 0.55] & 0.28 [0.26, 0.30] \\
         & Gemma-4               & 0.20 [0.18, 0.23] & 0.07 [0.05, 0.09] & 0.55 [0.53, 0.57] & 0.31 [0.29, 0.33] \\
\bottomrule
\end{tabular}
\end{table}

\begin{table}[h]
\centering
\footnotesize
\setlength{\tabcolsep}{3pt}
\caption{Table~\ref{tab:exact_by_k} with 95\% bootstrap intervals: exact-set
recovery (\%) under full information (F) and interaction (I), 200 cases per $k$.}
\label{tab:exact_by_k_ci}
\begin{tabular}{lllcccc}
\toprule
Source & Model & & $k{=}1$ & $k{=}2$ & $k{=}3$ & $k{=}4$ \\
\midrule
DDXPlus  & GPT-5.6             & F & 30 [24, 36] & 10 [6, 14] & 4 [2, 6] & 0 [0, 2] \\
         &                     & I & 27 [21, 33] & 6 [3, 10] & 0 [0, 2] & 0 [0, 0] \\
         & Gemini-3.8-Flash    & F & 76 [70, 82] & 23 [18, 29] & 4 [2, 8] & 0 [0, 2] \\
         &                     & I & 22 [16, 28] & 8 [4, 12] & 1 [0, 2] & 0 [0, 2] \\
         & Qwen3.8-Flash       & F & 44 [37, 51] & 22 [16, 28] & 4 [2, 7] & 0 [0, 2] \\
         &                     & I & 14 [10, 20] & 9 [6, 13] & 1 [0, 2] & 0 [0, 0] \\
         & GLM-5.3-Flash       & F & 21 [16, 27] & 11 [7, 16] & 8 [5, 12] & 2 [0, 5] \\
         &                     & I & 8 [5, 12] & 0 [0, 2] & 0 [0, 0] & 0 [0, 0] \\
         & DeepSeek-V4.1-Flash & F & 9 [5, 13] & 4 [2, 8] & 1 [0, 2] & 0 [0, 2] \\
         &                     & I & 4 [2, 8] & 1 [0, 2] & 0 [0, 2] & 0 [0, 2] \\
         & Gemma-4               & F & 26 [20, 32] & 10 [6, 14] & 4 [2, 6] & 2 [0, 4] \\
         &                     & I & 2 [0, 5] & 2 [0, 5] & 1 [0, 2] & 0 [0, 0] \\
\midrule
ePOCT+   & GPT-5.6             & F & 39 [32, 46] & 24 [18, 30] & 12 [8, 17] & 10 [6, 14] \\
         &                     & I & 18 [13, 24] & 7 [4, 10] & 4 [2, 6] & 4 [1, 6] \\
         & Gemini-3.8-Flash    & F & 34 [28, 40] & 25 [19, 32] & 15 [10, 20] & 10 [6, 14] \\
         &                     & I & 22 [16, 27] & 8 [5, 12] & 7 [4, 10] & 2 [0, 5] \\
         & Qwen3.8-Flash       & F & 34 [27, 40] & 26 [20, 32] & 10 [6, 14] & 4 [2, 7] \\
         &                     & I & 18 [13, 24] & 5 [2, 8] & 2 [0, 4] & 0 [0, 0] \\
         & GLM-5.3-Flash       & F & 24 [18, 30] & 20 [14, 26] & 10 [6, 14] & 8 [4, 12] \\
         &                     & I & 13 [8, 18] & 6 [2, 9] & 2 [0, 5] & 0 [0, 2] \\
         & DeepSeek-V4.1-Flash & F & 14 [10, 19] & 21 [16, 27] & 16 [11, 21] & 11 [7, 16] \\
         &                     & I & 13 [8, 18] & 6 [2, 9] & 6 [3, 10] & 0 [0, 2] \\
         & Gemma-4               & F & 32 [26, 38] & 30 [23, 36] & 15 [10, 20] & 5 [2, 8] \\
         &                     & I & 14 [10, 20] & 8 [5, 12] & 6 [3, 10] & 0 [0, 0] \\
\bottomrule
\end{tabular}
\end{table}

%% file: appendix/additional_results/perk.tex
\subsection{Per-count scores and outcomes}
\label{app:perk}

Table~\ref{tab:perk} breaks Table~\ref{tab:main} down by $k$, and
Figure~\ref{fig:app_outcomes} splits interactive cases by outcome. Most failures
are partially incorrect at every count. Models add wrong diagnoses rather than
return a correct subset. The predicted set grows with $k$ and exceeds the true
count up to $k=3$, and a correct count rarely comes with the correct set
(Figure~\ref{fig:app_count}).

\begin{table}[h]
\centering
\small
\caption{Per-count scores on DDXPlus for the six models. Each cell reads F1 / Jaccard / exact over
200 cases.}
\label{tab:perk}
\setlength{\tabcolsep}{4pt}
\begin{tabular}{llcccc}
\toprule
Condition & Model & $k{=}1$ & $k{=}2$ & $k{=}3$ & $k{=}4$ \\
\midrule
full & GPT-5.6 & 0.65/0.55/0.30 & 0.58/0.45/0.10 & 0.57/0.43/0.04 & 0.56/0.41/0.01 \\
full & Gemini-3.8-Flash & 0.79/0.79/0.77 & 0.58/0.48/0.23 & 0.47/0.35/0.04 & 0.40/0.28/0.01 \\
full & Qwen3.8-Flash & 0.70/0.64/0.44 & 0.57/0.47/0.21 & 0.53/0.41/0.04 & 0.48/0.34/0.01 \\
full & GLM-5.3-Flash & 0.61/0.50/0.21 & 0.57/0.45/0.11 & 0.59/0.46/0.09 & 0.57/0.43/0.03 \\
full & DeepSeek-V4.1-Flash & 0.50/0.37/0.09 & 0.51/0.37/0.04 & 0.50/0.36/0.01 & 0.50/0.36/0.01 \\
full & Gemma-4       & 0.63/0.53/0.26 & 0.56/0.43/0.10 & 0.52/0.38/0.04 & 0.53/0.39/0.01 \\
\midrule
interactive & GPT-5.6 & 0.54/0.46/0.27 & 0.38/0.29/0.06 & 0.33/0.23/0.01 & 0.30/0.20/0.00 \\
interactive & Gemini-3.8-Flash & 0.53/0.45/0.22 & 0.44/0.33/0.08 & 0.38/0.27/0.01 & 0.34/0.23/0.01 \\
interactive & Qwen3.8-Flash & 0.47/0.38/0.14 & 0.36/0.27/0.09 & 0.32/0.22/0.01 & 0.29/0.19/0.00 \\
interactive & GLM-5.3-Flash & 0.41/0.31/0.09 & 0.34/0.23/0.01 & 0.34/0.22/0.00 & 0.33/0.21/0.00 \\
interactive & DeepSeek-V4.1-Flash & 0.41/0.30/0.04 & 0.36/0.25/0.01 & 0.32/0.22/0.01 & 0.34/0.23/0.01 \\
interactive & Gemma-4       & 0.41/0.30/0.03 & 0.36/0.25/0.03 & 0.30/0.21/0.01 & 0.29/0.19/0.00 \\
\bottomrule
\end{tabular}
\end{table}

\begin{figure}[h]
\centering
\includegraphics[width=0.9\linewidth]{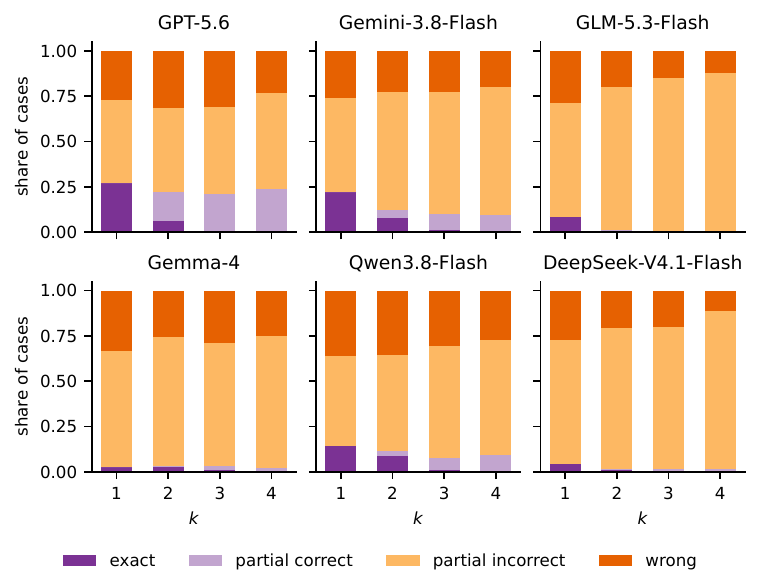}
\caption{Interactive outcomes by $k$ (DDXPlus, 200 cases per $k$). A partially
correct set is a subset of the truth; a partially incorrect one adds wrong
labels.}
\label{fig:app_outcomes}
\end{figure}

\begin{figure}[h]
\centering
\includegraphics[width=0.75\linewidth]{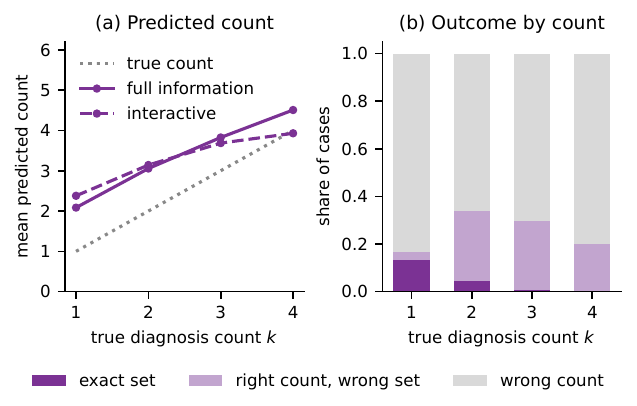}
\caption{Set size (DDXPlus, six models, 200 cases per $k$). (a) Predicted count
against $k$. (b) Exact set, right count with a wrong set, or wrong count.}
\label{fig:app_count}
\end{figure}

%% file: appendix/additional_results/oracle.tex
\subsection{Oracle count}
\label{app:oracle}

With the true $k$ supplied (Section~\ref{sec:gap}), exact-set recovery still
collapses with $k$ (Figure~\ref{fig:app_oracle}, left). The oracle-count runs
use the first 50 cases per $k$ of the main evaluation. Relative to standard
interaction, the count helps at $k=1$ but has mixed effects at $k\ge2$ (right).

\begin{figure}[h]
\centering
\includegraphics[width=0.47\linewidth]{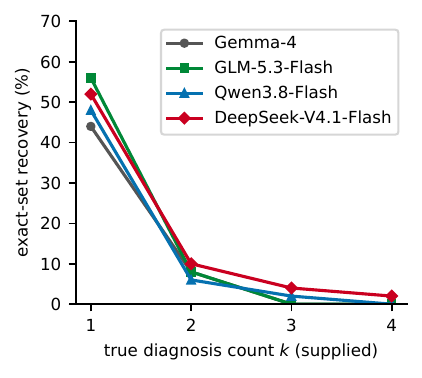}
\hfill
\includegraphics[width=0.49\linewidth]{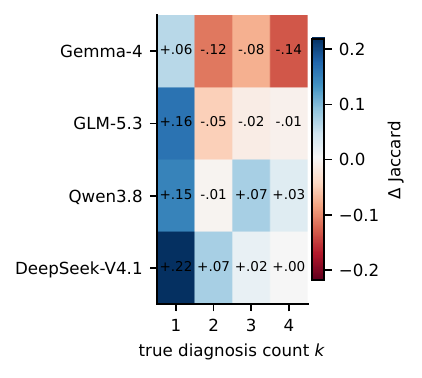}
\caption{Oracle count (50 cases per $k$). Left, exact-set recovery with the true
$k$ supplied. Right, change in Jaccard against standard interaction.}
\label{fig:app_oracle}
\end{figure}

%% file: appendix/additional_results/interaction.tex
\subsection{Over-calling and the interaction between count and condition}
\label{app:interaction}

\paragraph{Precision and recall.} Averaged over the six models, precision is
flat across $k$ in both conditions, near 55\% with full information and 35\%
in interaction, while recall falls from 89\% to 55\% and from 70\% to 32\%
(Figure~\ref{fig:app_pr}). The loss at $k=1$ therefore comes from extra labels,
since models return about two diagnoses per single-diagnosis case. As diagnoses
co-occur, the loss moves to missed diagnoses.

\begin{figure}[h]
\centering
\includegraphics[width=0.45\linewidth]{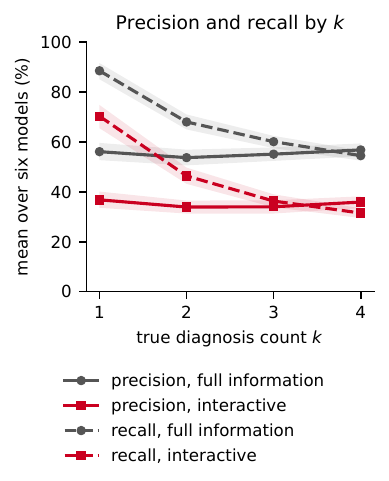}
\caption{Precision and recall by $k$ on DDXPlus, mean of six models (200 cases per
$k$).}
\label{fig:app_pr}
\end{figure}

\paragraph{Interaction test.}
For each model we compute the fall in mean Jaccard from $k=1$ to the mean
over $k\ge2$, in interaction minus with full information. A positive value
means interaction steepens the fall with $k$. Intervals are 95\% case
bootstraps stratified by $k$ (4,000 draws), resampling the same cases for
every model and condition. Exact-set recovery is not used, because it sits
near zero at $k\ge2$ in both conditions and would force the contrast towards
zero. Averaged over the six models, Jaccard falls with $k$ by a similar
amount in both conditions (Figure~\ref{fig:app_jaccard}). The per-model
estimates have mixed signs and the pooled value is close to zero
(Table~\ref{tab:interaction}).

\begin{figure}[h]
\centering
\includegraphics[width=0.49\linewidth]{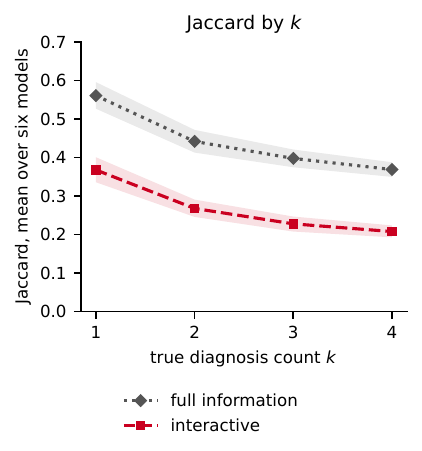}
\caption{Jaccard by $k$ on DDXPlus, mean of six models (200 cases per $k$).}
\label{fig:app_jaccard}
\end{figure}

\begin{table}[h]
\centering
\small
\caption{Count $\times$ condition interaction on Jaccard, DDXPlus, 200 cases per $k$.}
\label{tab:interaction}
\begin{tabular}{lc}
\toprule
Model & Interaction [95\% CI] \\
\midrule
GPT-5.6             & $+0.108$ $[+0.049, +0.166]$ \\
Gemini-3.8-Flash    & $-0.242$ $[-0.303, -0.185]$ \\
GLM-5.3-Flash       & $+0.033$ $[-0.015, +0.081]$ \\
Gemma-4               & $-0.038$ $[-0.087, +0.010]$ \\
Qwen3.8-Flash       & $-0.076$ $[-0.134, -0.021]$ \\
DeepSeek-V4.1-Flash & $+0.065$ $[+0.027, +0.104]$ \\
\midrule
Pooled              & $-0.025$ $[-0.053, +0.002]$ \\
\bottomrule
\end{tabular}
\end{table}

%% file: appendix/additional_results/single.tex
\subsection{Single-diagnosis control}
\label{app:single}

Asked for one diagnosis, Gemma names a true one in roughly half of the cases at
every $k$ (Table~\ref{tab:single}), while it recovers the exact set in few cases
with or without the true $k$ (Figure~\ref{fig:app_gap_gemma}). Among these hits, the guess is the
presenting-complaint disease in 68\% of cases at $k=2$, 82\% at $k=3$ and
43\% at $k=4$, against 50\%, 33\% and 25\% if every true diagnosis were
equally likely to be named.

\begin{table}[h]
\centering
\small
\caption{Single-diagnosis instruction for Gemma across $k$, 50 cases per $k$. hit@1 is the fraction
of cases where the one guess names any true diagnosis. chief is the fraction
where the guess is the presenting-complaint disease.}
\label{tab:single}
\begin{tabular}{lcccc}
\toprule
 & $k{=}1$ & $k{=}2$ & $k{=}3$ & $k{=}4$ \\
\midrule
hit@1          & 0.54 & 0.56 & 0.34 & 0.42 \\
mean count     & 1.00 & 1.00 & 1.00 & 1.00 \\
chief selected & --   & 0.38 & 0.28 & 0.18 \\
$\tau$         & 11.4 & 11.5 & 10.9 & 11.3 \\
\bottomrule
\end{tabular}
\end{table}

\begin{figure}[h]
\centering
\includegraphics[width=0.49\linewidth]{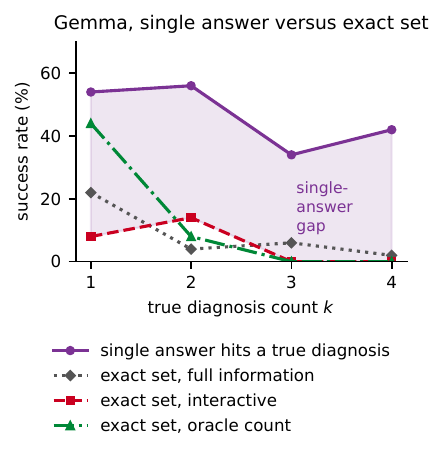}
\caption{Gemma on DDXPlus (50 cases per $k$). Single-answer hits versus exact
sets with full information, in interaction and with the oracle count.}
\label{fig:app_gap_gemma}
\end{figure}

%% file: appendix/additional_results/evidence_support.tex
\subsection{Evidence acquisition and use: supporting analyses}
\label{app:evidence_support}

These analyses extend Section~\ref{sec:tree} on the matched respiratory cohort
(Appendix~\ref{app:evidence_matched}). Table~\ref{tab:focal} scores one fixed
label: a diagnosis alone versus the same diagnosis with a second disease added.
With the full record or the tree's findings, the fixed reader loses it in 3.6
and 2.4 points more cases, GPT-5.6 in 8.3 and 11.3. GPT-5.6's excess loss over
the reader is resolved with the tree's findings ($-8.9$ points,
$[-16.1,-1.8]$) but not with the full record ($-4.8$, $[-11.3,1.2]$).
Table~\ref{tab:evidence_robust} shows that the growing cost of
GPT-5.6's questions holds for other metrics and without the two rhinosinusitis
labels, whose symptom schemas nearly coincide. The direct tree-versus-GPT-5.6 gap
does not clearly widen with $k$, because it mixes questioning and diagnosis.

\begin{table}[h]
\centering
\small
\caption{Recovery (\%) of the same focal diagnosis alone ($k=1$) and with a
second disease ($k=2$, opening on the focal diagnosis), with the paired change
in points.}
\label{tab:focal}
\begin{tabular}{llccc}
\toprule
Evidence & Reader & Alone & With second & Change \\
\midrule
Full record        & GPT-5.6      & 85.7 & 77.4 & $-8.3\;[-13.7,-3.0]$ \\
Full record        & Fixed reader & 96.4 & 92.9 & $-3.6\;[-6.5,-0.6]$ \\
Tree's findings    & GPT-5.6      & 81.0 & 69.6 & $-11.3\;[-17.9,-4.8]$ \\
Tree's findings    & Fixed reader & 92.9 & 90.5 & $-2.4\;[-6.0,1.2]$ \\
GPT-5.6's findings & GPT-5.6      & 70.2 & 61.3 & $-8.9\;[-15.5,-2.4]$ \\
GPT-5.6's findings & Fixed reader & 91.7 & 75.6 & $-16.1\;[-21.4,-10.7]$ \\
\bottomrule
\end{tabular}
\end{table}

\begin{table}[h]
\centering
\small
\caption{Change from $k=1$ to $k=2$ in the tree's advantage over GPT-5.6.}
\label{tab:evidence_robust}
\begin{tabular}{lc}
\toprule
Contrast & $k=2$ minus $k=1$ \\
\midrule
Fixed reader, exact set (Figure~\ref{fig:tree}b) & $+11.9$ pp $[3.6,20.2]$ \\
Fixed reader, label recall & $+12.2$ pp $[7.1,17.6]$ \\
Fixed reader, label log loss & $0.62$ bits $[0.31,0.94]$ \\
Fixed reader, exact set, without rhinosinusitis & $+13.8$ pp $[2.5,23.8]$ \\
Direct tree versus GPT-5.6, exact set & $+7.1$ pp $[-3.0,16.7]$ \\
\bottomrule
\end{tabular}
\end{table}

%% file: appendix/additional_results/designed_k2.tex
\subsection{Designed $k=2$ cohort}
\label{app:designed_k2}

\paragraph{Design.} The generator opens every evaluated $k=2$ case with the
more severe diagnosis, so in the main cohort opening and severity are
confounded. The designed cohort holds severity equal. We sample pairs of
equally severe DDXPlus diagnoses whose symptom profiles, without antecedents,
overlap little (Jaccard at most 0.10) or much (at least 0.30), leaving out
pairs screened as infeasible or unlikely. We realise one patient per pair with
the generator's record-merging rules, 195 patients covering 47 of the 49
diagnoses. Each patient is consulted twice. The opening lists the initial
evidence of one diagnosis and two findings exclusive to it within the pair,
and nothing else changes. Consultations use the 20-question budget and end at
the natural stop, without probes or forced continuation, and background
complaints are withheld. GPT-5.6 completed all 390 consultations and
Gemma-4 389. When a model follows a diagnosis with advice or corrects
itself, we keep the last diagnosis it names. Intervals are 95\% bootstrap
intervals over patients, which are also diagnosis pairs.

\begin{table}[h]
\centering
\small
\caption{Designed $k=2$ cohort, both openings pooled, with 95\% intervals.}
\label{tab:designed_k2}
\begin{tabular}{lcc}
\toprule
& GPT-5.6 & Gemma-4 \\
\midrule
Recall of A & 0.65 $[0.60, 0.69]$ & 0.58 $[0.52, 0.62]$ \\
Recall of B & 0.32 $[0.27, 0.36]$ & 0.41 $[0.37, 0.46]$ \\
A minus B, paired, points & $+33$ $[27, 38]$ & $+16$ $[12, 21]$ \\
\midrule
Final set holds A and B & 0.17 $[0.13, 0.22]$ & 0.24 $[0.19, 0.29]$ \\
Final set holds A only & 0.47 $[0.43, 0.52]$ & 0.33 $[0.29, 0.37]$ \\
Final set holds B only & 0.15 $[0.12, 0.18]$ & 0.17 $[0.14, 0.20]$ \\
Final set holds neither & 0.21 $[0.16, 0.25]$ & 0.26 $[0.21, 0.31]$ \\
\bottomrule
\end{tabular}
\end{table}

\paragraph{Reading.} Both models recover the opening diagnosis more often,
and the effect is twice as large for GPT-5.6. In both, the final set holds A
alone more often than both diagnoses.

%% file: appendix/additional_results/sharing.tex
\subsection{Symptom-sharing cohort}
\label{app:sharing}

\paragraph{Why.} The designed cohort of Appendix~\ref{app:designed_k2} cannot show what a
shared symptom does when it is in view. The two findings after the first
complaint are exclusive to A by construction, and a specific symptom of both
diagnoses opened only 6 of its 200 high-overlap consultations. Its high
overlap also comes mostly from the multi-part pain question, with 2.3 shared
questions per pair of which 0.4 are specific, as defined below.

\paragraph{Design.} We take pairs of equally severe DDXPlus diagnoses whose
symptom lists share at least one specific symptom, one listed by at most 10
of the 49 profiles when a question and its sub-questions count once. We leave
out pairs screened as infeasible or unlikely, pairs whose age windows do not
overlap, and nested pairs such as URTI with influenza, where naming one of
the two can be right. For each pair we realise six patients with the
generator's record-merging rules, drawing A at random, and keep a record only
if it holds, besides A's first finding, a specific and a general yes-or-no
symptom that both diagnoses list, the same for A alone, and a specific
symptom of B alone. This gives 121
patients from 21 pairs. Each patient is consulted twice with the same record,
the opening giving A's first finding followed by the two shared symptoms or
by the two symptoms of A alone. As a control, the 88 patients whose record
allows it are also consulted with A's first finding followed by a specific
and a second symptom of B alone (general in 28), which shows how much a
symptom of B raises B. Nothing is removed from the record, so the patient simulator answers
alike under every opening. Prompts, budget and scorer are those of the designed
cohort, and consultations end at the natural stop. We analyse the 97 patients
whose diagnoses have severity 3 to 5, of whom 65 have all three openings.

\paragraph{Development and analysis.} Two earlier versions failed in small
runs. Removing a shared finding from the record leaked, since the patient
simulator still reported it, and opening with the more severe diagnosis
ended most consultations within a few questions, with B named in none of 16. We
fixed each analysis before its full run, with B named as the primary outcome,
first for shared minus A's symptoms and then for shared minus B's symptoms.
Intervals are 95\% bootstrap intervals over disease pairs.

\begin{table}[h]
\centering\small
\caption{Symptom-sharing cohort, GPT-5.6, diagnoses of severity 3 to 5.
Differences between openings in points, with 95\% intervals.}
\label{tab:sharing}
\begin{tabular}{lcccc}
\toprule
Opening contrast & Patients & B named & A named & Both named \\
\midrule
Shared minus A's symptoms & 97 & $+13$ $[4, 24]$ & $-4$ $[-12, 4]$ & $+5$ $[-4, 14]$ \\
Shared minus B's symptoms & 65 & $-5$ $[-17, 8]$ & $+9$ $[2, 17]$ & $+2$ $[-7, 11]$ \\
B's minus A's symptoms & 65 & $+17$ $[8, 27]$ & $-12$ $[-22, -3]$ & $+2$ $[-3, 7]$ \\
\bottomrule
\end{tabular}
\end{table}

\begin{figure}[h]
\centering
\includegraphics[width=0.55\linewidth]{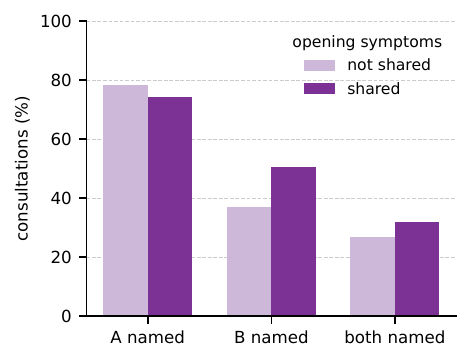}
\caption{Symptom-sharing cohort, GPT-5.6, diagnoses of severity 3 to 5. The
same patients open with A's first finding and two symptoms of A alone (not
shared) or two that both diagnoses share. Share of consultations naming A,
B and both.}
\label{fig:app_sharing_bars}
\end{figure}

\begin{figure}[h]
\centering
\includegraphics[width=0.55\linewidth]{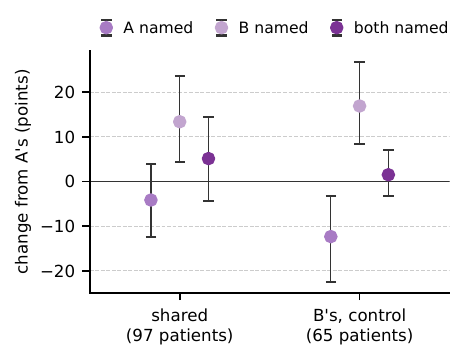}
\caption{Symptom-sharing cohort, GPT-5.6. Change from the opening with A's
symptoms within the same patients, for shared symptoms (the patients of
Figure~\ref{fig:app_sharing_bars}) and for B's symptoms as a control (the patients who
also had that opening), with 95\% intervals over disease pairs.}
\label{fig:app_sharing_paired}
\end{figure}

\paragraph{Reading.} On B, shared symptoms behave much closer to symptoms of
B than to symptoms of A. In the 24 patients whose B opening matches the shared one in
kind, the difference is $+8$ points $[-14, 25]$. Of the net gain in B, about 5 points
name both diagnoses and 8 name B without A. The share naming both rises only
from 27\% to 32\%, not reliably.

Exact-set recovery changes little, 5\% and 8\%.

\paragraph{Main consultations.} In the main DDXPlus consultations with
$k \ge 2$, the chief complaint is in B's profile for 31\% of the second
diagnoses. Within the same diagnosis B, B is then named more often, by 5 to 13
points, for all six models under both conditions, but without that adjustment
the difference is reliable in only 2 of the 12. The pattern also appears with
the full record, so it may reflect related diagnoses rather than the opening.

\paragraph{Exclusive diagnoses.} If the model treated the two diagnoses as
mutually exclusive, symptoms of A alone would lower the odds of B against A and
shared symptoms would leave them unchanged, so this belief also predicts more B
with shared symptoms. The cohort therefore does not separate an exclusive from
a separate reading. Symptoms of B alone lower A ($-12$ points $[-22, -3]$), but
they also leave the opening with one finding of A instead of three.

\paragraph{Limits.} The cohort covers 21 disease pairs and one doctor model,
and its openings are built for the test. The rise in the share naming both
diagnoses is not reliable at this size. With 21 pairs, one doctor model and openings
built for the test, this is a controlled probe and not a population estimate.

%% file: appendix/additional_results/probe_anchoring.tex
\subsection{Anchoring in the main consultations}
\label{app:probe_anchoring}

The designed cohort amplifies anchoring with a three-finding opening. We also
look for it in the main DDXPlus consultations, whose opening is the single
presenting complaint. There the generator opens with the most severe
diagnosis, so we keep only consultations in which another target has the
same severity as the opening diagnosis: 54 cases at $k\ge2$, 216
consultations over the four open models. The probe at each turn is a fresh
call that reads the transcript so far and never sees the doctor's
hypotheses, so it measures what the collected evidence supports.

The opening diagnosis pulls ahead within the first questions and stays ahead
(Figure~\ref{fig:probe_anchoring}). The gap between the opening diagnosis and
an equally severe second diagnosis is $8$ points at the opening
($[3, 13]$) and $16$ points when the doctor stops ($[2, 29]$). The opening
diagnosis first appears at a median of four questions, the other diagnosis at
eleven, and the other diagnosis never appears before the stop in 46\%
($[37, 55]$) of consultations, against 31\% ($[23, 38]$) for the opening
diagnosis. In these runs the doctor stops at a median of twelve of its twenty
questions, so the budget does not bind. The gap still remains at the cap, after forced
continuation. Since the probe is neutral, the widening comes from the questions
the doctor chose: the evidence it collects follows the opening diagnosis.

\begin{figure}[h]
\centering
\includegraphics[width=0.55\linewidth]{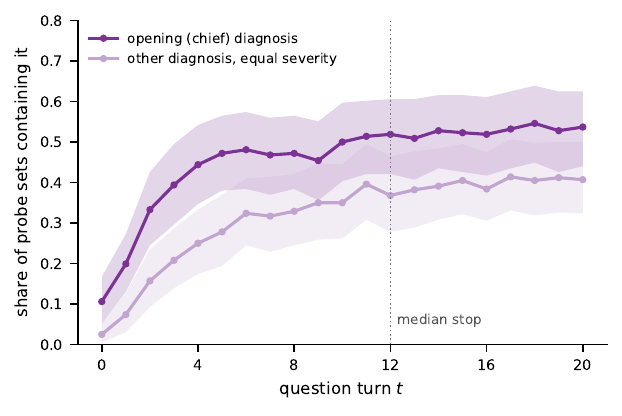}
\caption{Share of probe sets holding the opening diagnosis and an equally severe
second one (DDXPlus, four open models, 50 cases per $k$). Bands are 95\%
bootstrap intervals.}
\label{fig:probe_anchoring}
\end{figure}

%% file: appendix/additional_results/compliance.tex
\subsection{Post-stop behaviour across models}
\label{app:compliance}

The four open models show the pattern GPT-5.6 shows in
Section~\ref{sec:stopping} (Table~\ref{tab:compliance}). Their probe runs, like
those of Appendix~\ref{app:probe_anchoring}, predate the main evaluation and
place the question counter in the system message
(Appendix~\ref{app:sim:prompts}). There the open models stop at a median of 8
to 14 questions; in the main evaluation, with the counter after each answer,
they stop at 17 to 20 and often use the whole budget. GPT-5.6 stops at a median
of seven in both, and its probe run follows the main protocol. GPT-5.6 was probed
after its run on the stored transcripts with the same probe prompt; a probe
sees only the dialogue prefix, so this is the same measurement. Under the same probe prompt, F1 at the cap is no
higher than at the natural stop, and models differ mainly in how often they
refuse to keep asking. The probe set grows after the stop, but mostly with
wrong diagnoses (Figure~\ref{fig:app_composition}). At $k=1$ further questions
lower exact-set recovery, from 25\% after three questions to 15\% at the cap;
at $k\ge3$ a probe almost never holds the complete set
(Figure~\ref{fig:app_exact}). Aligned at the natural stop, neither probe
Jaccard nor exact-set recovery changes over six forced questions
(Figure~\ref{fig:app_aligned}); Figure~\ref{fig:app_stop_models} shows each
model.

\paragraph{Sampling nulls.} A probe is a single stochastic readout, so statistics
taken over many probes grow with the number of draws. Two such statistics do
not survive a null. The share of true diagnoses that appear in any probe rises
after the stop by 0.066 per label, and by 0.061 for the $k$ wrong labels whose
DDXPlus profiles best match the case (difference $0.005$, $[-0.013, 0.024]$).
The best prefix F1 (0.62 pooled) is close to the maximum over the same number
of draws from a readout that learns nothing, fitted to the probes after the
stop (0.54); the excess is 0.08 ($[0.07, 0.09]$). Of the true diagnoses found
before the stop, 18\% are absent from the probe at the stop under the same
prompt, and 7\% from every later probe. We therefore report single probes in
the main text.

\begin{table}[h]
\centering
\small
\caption{Post-stop readouts. natural is the F1 of the voluntary committed
answer; stop and cap are the F1 of the single probe at the natural stop and at
the last forced question. Refusals count post-stop turns where the model
committed instead of asking. Four open models, 50 cases per $k$.}
\label{tab:compliance}
\begin{tabular}{lccccc}
\toprule
Model & natural & probe at stop & probe at cap & refusals & forced turns \\
\midrule
Gemma-4    & 0.426 & 0.457 & 0.424 & 404  & 7.6 \\
GLM      & 0.401 & 0.387 & 0.363 & 1126 & 7.1 \\
Qwen     & 0.309 & 0.357 & 0.373 & 64   & 5.5 \\
DeepSeek & 0.354 & 0.381 & 0.359 & 682  & 8.9 \\
\bottomrule
\end{tabular}
\end{table}

\begin{figure}[h]
\centering
\includegraphics[width=\linewidth]{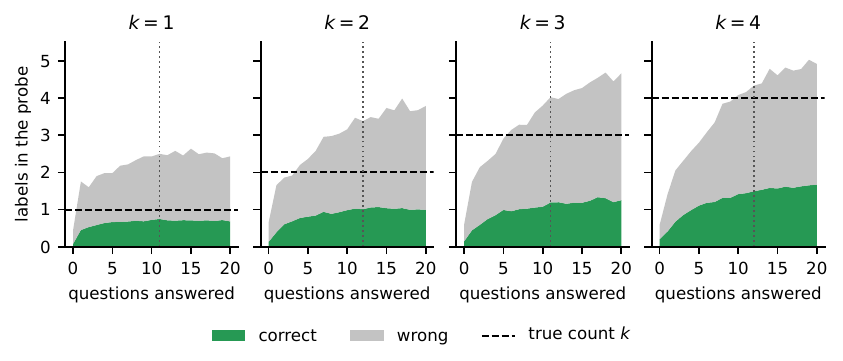}
\caption{Composition of a single probe by question turn (four open models, 50
cases per $k$). Dashed lines mark $k$, dotted lines the median stop.}
\label{fig:app_composition}
\end{figure}

\begin{figure}[h]
\centering
\includegraphics[width=0.48\linewidth]{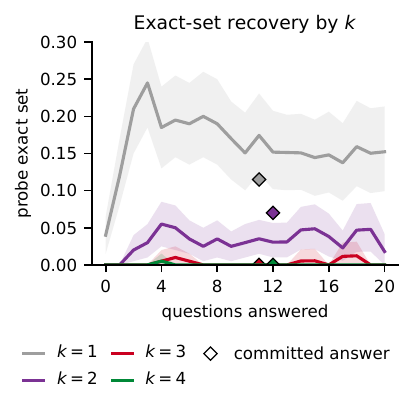}\hfill
\includegraphics[width=0.48\linewidth]{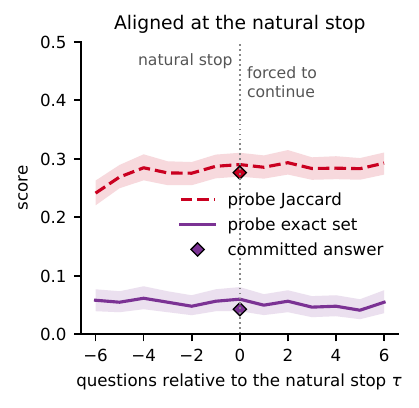}
\caption{Left, single-probe exact-set recovery by $k$. Right, probes aligned at the
natural stop $\tau$ (586 consultations). Four open models, 50 cases per $k$.}
\label{fig:app_exact}
\label{fig:app_aligned}
\end{figure}

\begin{figure}[h]
\centering
\includegraphics[width=\linewidth]{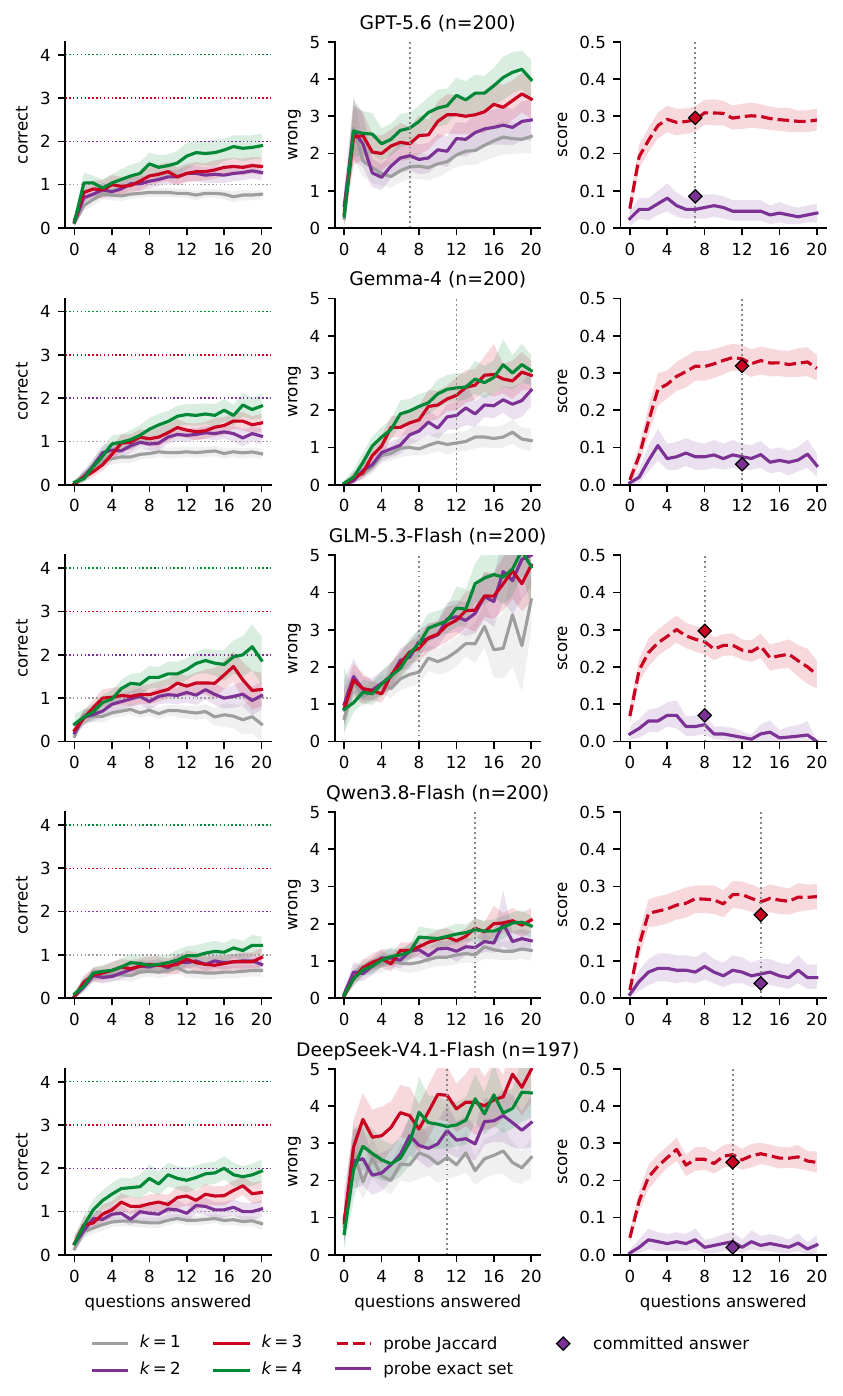}
\caption{Figure~\ref{fig:retention} per model (50 cases per $k$). Probes listing
40 or more labels are excluded.}
\label{fig:app_stop_models}
\end{figure}

%% file: appendix/theory.tex
\section{Formal Framework}
\label{app:theory}

This appendix states the task formally and contrasts two idealised
reference doctors, a \emph{set tracker}, which treats every diagnosis as
its own hypothesis, and a \emph{single-hypothesis tracker}, which carries
one working diagnosis at a time. It does not model how a language model
computes its answer. It states what each reference implies and which of
these implications our data can test.

\input{appendix/theory/setup.tex}

\input{appendix/theory/reference.tex}

\input{appendix/theory/results.tex}

%% file: appendix/theory/setup.tex
\subsection{Setup}
\label{app:theory:setup}

\begin{definition}[Cases and consultations]
\label{def:labels}
\label{def:process}
A case has a target $\mathcal{D} \subseteq \mathcal{L}$ with
$k = |\mathcal{D}| \ge 1$ and a record $x \subseteq \mathcal{F}$ of
findings. Each $d \in \mathcal{D}$ contributes the findings $v_d
\subseteq x$ of its path or source record (Section~\ref{sec:method}).
The \emph{informative set} $Q_d$ of a label contains the findings the
source lists for it, namely the variables on its ePOCT+ paths or the evidences
DDXPlus lists for the pathology. The doctor sees $\mathcal{L}$ and the
opening, whose findings form $O \subseteq x$. At each turn it asks for
findings and receives $\mathbf{1}\{q \in x\}$ for each finding $q$ asked,
or stops at $\tau \le T = 20$; it then names a set
$\widehat{\mathcal{D}}$. Under full information the whole record is
disclosed and $\tau = 0$.
\end{definition}

\begin{definition}[Scores]
\label{def:observables}
Per case, $u = |\mathcal{D} \cap \widehat{\mathcal{D}}|$ is the number of
correct diagnoses and $\widehat k = |\widehat{\mathcal{D}}|$ the
predicted count. Over cases with target $\mathcal{D}$, the per-label
recall is $a^{\mathcal{D}}_d = P(d \in \widehat{\mathcal{D}})$, and
$\alpha_d = a^{\{d\}}_d$ is the recall of $d$ when it occurs alone.
$P(u \ge 1)$ is the rate a single-answer evaluation credits and
$P(\widehat{\mathcal{D}} = \mathcal{D})$ the exact-set rate.
\end{definition}

\begin{definition}[Grounding]
\label{def:grounding}
A true label $d$ is \emph{grounded} at $t$ if $Q_d$ meets the findings
disclosed so far, which are the opening and the findings asked in interaction, or
the whole record under full information. $G_t \subseteq \mathcal{D}$ is
the set of grounded true labels and $m_0 = |G_0|$.
\end{definition}

\begin{assumption}[Grounded recovery]
\label{as:grounded}
A true label that is not grounded when the doctor stops is named with
probability at most $\beta_d$, its rate of being named without evidence,
$P(d \in \widehat{\mathcal{D}} \mid d \in \mathcal{D},\ d \notin G_\tau)
\le \beta_d$.
\end{assumption}

The opening counts as evidence, so the diagnosis that supplies it can be
named without any question. Under full information every true label is
grounded and the assumption is void.

%% file: appendix/theory/reference.tex
\subsection{Two Reference Doctors}
\label{app:theory:reference}

\begin{definition}[Set tracker]
\label{def:settracker}
A set tracker satisfies (N), \emph{no interference}, which means
$a^{\mathcal{D}}_d = \alpha_d$ for every target $\mathcal{D} \ni d$. A
label is named as often in company as alone.
\end{definition}

\begin{definition}[Single-hypothesis tracker]
\label{def:concentrated}
A single-hypothesis tracker carries a working hypothesis $W_t$, set by
the opening at $t = 0$, and satisfies
\begin{itemize}
\item[(X)] \emph{exclusive belief}. Its support is a distribution $s_t$
  on $\mathcal{L}$ with $W_t = \arg\max_d s_t(d)$, and it names the
  labels with $s_t(d) \ge \theta$;
\item[(C)] \emph{concentration}. It asks only findings in $Q_{W_t}$.
\end{itemize}
\end{definition}

%% file: appendix/theory/results.tex
\subsection{Results}
\label{app:theory:results}

\begin{lemma}[Single-answer gap]
\label{lem:sandwich}
For any doctor and target,
$P(\widehat{\mathcal{D}} = \mathcal{D})
\le \min_{d \in \mathcal{D}} a^{\mathcal{D}}_d
\le \max_{d \in \mathcal{D}} a^{\mathcal{D}}_d \le P(u \ge 1)$.
At $k = 1$ the two ends differ by
$P(d \in \widehat{\mathcal{D}},\ \widehat k > 1)$.
\end{lemma}

\begin{proof}
Per case,
$\mathbf{1}\{\widehat{\mathcal{D}} = \mathcal{D}\} \le
\mathbf{1}\{d \in \widehat{\mathcal{D}}\} \le \mathbf{1}\{u \ge 1\}$ for
every $d \in \mathcal{D}$; take expectations. At $k = 1$, $u \ge 1$
without an exact match means $d$ is named with extra labels.
\end{proof}

A single-answer score thus hides extra labels at $k = 1$ and, at
$k \ge 2$, at least the spread of recall across the true labels.

\begin{lemma}[Count and identity]
\label{lem:count}
For any doctor, $u \le \min(k, \widehat k)$ and
$P(\widehat{\mathcal{D}} = \mathcal{D}) = P(\widehat k = k)\,
P(\widehat{\mathcal{D}} = \mathcal{D} \mid \widehat k = k)$.
\end{lemma}

\begin{proof}
$u$ counts elements of both sets, and an exact match has the right
count.
\end{proof}

On DDXPlus with full information, the six models return the right count
in $41$, $28$, $19$ and $14\%$ of cases at $k = 1, \dots, 4$, and among
these the set is right in $84$, $47$, $22$ and $7\%$. Both factors fall
with $k$. A disclosed count can fix only the first, unless it also
changes which labels are named (Appendix~\ref{app:oracle}).

\begin{proposition}[Set tracker]
\label{prop:oneshot}
Under (N), $\mathbb{E}[u \mid \mathcal{D}] = \sum_{d \in \mathcal{D}}
\alpha_d$, and the recall of a label does not change when another label
is added to its target.
\end{proposition}

\begin{proof}
$u = \sum_{d \in \mathcal{D}} \mathbf{1}\{d \in
\widehat{\mathcal{D}}\}$; take expectations and apply (N).
\end{proof}

Exact-set recovery can fall with $k$ even under (N), since every label must
be named and no wrong one added; recall should not. On DDXPlus, (N)
predicts about $4 \times 0.885 = 3.5$ correct diagnoses at $k = 4$ with
full information and $4 \times 0.703 = 2.8$ in interaction; the models
find $2.2$ and $1.3$ (Appendix~\ref{app:interaction}). This pooled check
assumes that labels at $k = 4$ are as easy alone as those at $k = 1$.
Proposition~\ref{prop:separate} gives a paired test that does not.

\begin{proposition}[Separate hypotheses imply no interference]
\label{prop:separate}
\label{def:separate}
Let a doctor hold the diagnoses as \emph{separate hypotheses}, meaning that under full
information, whether it names $d$ depends only on the findings of the
record in $Q_d$. Let no other component meet $Q_d$. If $d$ is scored
alone and with other diagnoses, on records built from the same component
$v_d$ and without background, the doctor names $d$ equally often in
both cases.
\end{proposition}

\begin{proof}
Both records meet $Q_d$ in the same findings, $v_d \cap Q_d$, and the
decision on $d$ depends only on these.
\end{proof}

Read the other way, if adding a second disease lowers the recall of the
first, the decision on the first depends on findings of the second. The
matched respiratory cohort meets the conditions except for shared
findings. Each family scores the same components alone and combined,
without background, with the opening on the scored diagnosis
(Appendix~\ref{app:evidence_matched}). There a fixed per-label reader
loses $3.6$ points with the full record, so shared findings matter little,
while GPT-5.6 loses $8.3$ with the full record and $11.3$ with the tree's
findings (Table~\ref{tab:focal}). Its excess loss over the reader is
resolved with the tree's findings but not with the full record. The model
does not hold co-occurring diagnoses as separate hypotheses; with the
complete record the evidence for this is weaker.

\begin{proposition}[Exclusive belief]
\label{prop:exclusive}
Let (X) hold.
\begin{itemize}
\item[(i)] At most $\lfloor 1/\theta \rfloor$ labels are named. If a
  label $A$ has support at least $1 - \varepsilon$ and new evidence gives
  another label support $s'$, then $A$ loses at least
  $s' - \varepsilon$.
\item[(ii)] If $s_t(W_t) \ge 1 - \varepsilon$ with
  $\varepsilon \le \tfrac12$, every finding $q$ has expected information
  gain, under the doctor's own belief over a single label $Y$,
  $I(Y; \mathbf{1}\{q \in x\} \mid h_t) \le h_2(\varepsilon) +
  \varepsilon \log_2(|\mathcal{L}| - 1)$, where $h_2$ is binary entropy
  and $h_t$ the transcript at $t$.
\end{itemize}
\end{proposition}

\begin{proof}
(i) Support sums to one, so at most $1/\theta$ labels hold $\theta$
each, and $A$ keeps at most $1 - s'$. (ii) By the grouping property of
entropy, $H(Y \mid h_t) \le h_2(1 - s_t(W_t)) + (1 - s_t(W_t))
\log_2(|\mathcal{L}| - 1)$, which is at most the bound since $h_2$
increases on $[0, \tfrac12]$; the gain of any finding is at most
$H(Y \mid h_t)$.
\end{proof}

Under (X) a second diagnosis gains support only at the expense of the
first. By (ii), once the working diagnosis is settled no finding is worth
much to the doctor, so it has no reason to look for a second one.

\begin{proposition}[Grounding in interaction]
\label{prop:concentration}
Under Assumption~\ref{as:grounded},
$\mathbb{E}[u] \le \mathbb{E}|G_\tau| + \sum_{d \in \mathcal{D}}
\beta_d$. If no finding asked lies in $Q_d$ for a true label
$d \notin G_0$, then $G_\tau = G_0$ and
$\mathbb{E}[u] \le m_0 + \sum_{d \in \mathcal{D} \setminus G_0} \beta_d$
for every $\tau \le T$.
\end{proposition}

\begin{proof}
$u \le |G_\tau| + \sum_{d \in \mathcal{D}}
\mathbf{1}\{d \in \widehat{\mathcal{D}},\ d \notin G_\tau\}$, and each
term of the sum has expectation at most $\beta_d$. If no finding asked
meets $Q_d$, $d$ is grounded at $\tau$ only if it was at $0$.
\end{proof}

More questions thus raise the correct count only if they reach the
informative set of a missing diagnosis, which under (C) requires the
working hypothesis to move. In the main consultations an equally severe
second diagnosis never enters a probe set before the stop in 46\%
of consultations (Appendix~\ref{app:probe_anchoring}).